\documentclass{article}

\PassOptionsToPackage{authoryear}{natbib}
\usepackage[preprint]{style}

\usepackage[utf8]{inputenc}
\usepackage[T1]{fontenc}
\usepackage{amsmath,amsfonts,amssymb,amsthm}
\usepackage{graphicx}
\usepackage{booktabs}
\usepackage{microtype}
\usepackage{xcolor}
\usepackage{url}
\usepackage{algorithm}
\usepackage{algorithmic}
\usepackage{float}
\usepackage{placeins}
\usepackage{hyperref}

\newtheorem{theorem}{Theorem}[section]
\newtheorem{lemma}{Lemma}[section]
\newtheorem{assumption}{Assumption}[section]
\newtheorem{definition}{Definition}[section]
\newtheorem{remark}{Remark}[section]
\newtheorem{corollary}{Corollary}[theorem]
\newtheorem{proposition}{Proposition}
\newcommand{\D}{\mathcal{D}}

\newcommand{\ignore}[1]{}

\def\bold0{\mathbf{0}}

\def\al#1\eal{\begin{align}#1\end{align}}
\def\als#1\eals{\begin{align*}#1\end{align*}}

\renewcommand{\L}{\mathcal{L}}
\newcommand{\E}{\mathbb{E}}

\hypersetup{colorlinks,
linkcolor=blue,
citecolor=blue,
urlcolor=magenta,
linktocpage,
plainpages=false}

\title{SP-CACW: Convergence-Aware Client Weighting for Selfish Personalized
Learning}

\author{%
  Yaron Kiselman and Kfir Y. Levy\\
  Technion\\
  Haifa, Israel \\
}

\begin{document}

\maketitle

\begin{abstract} 
Collaborative learning is sustainable only when it benefits each participant. Standard federated learning optimizes a global average objective, which can under perform for clients whose data distributions differ substantially from the population. We study \emph{selfish personalization}: how a designated target client can use peer gradients to minimize its own risk while avoiding negative transfer. We propose \emph{SP-CACW}, a convergence-aware client-weighting framework that selects aggregation weights by minimizing an upper bound on the target client's convergence error. The resulting rule explicitly trades off peer bias against stochastic variance and can assign zero weight to harmful peers. We provide convergence guarantees under smoothness and bounded-variance assumptions and evaluate the method on MNIST, CIFAR-100, and LEAF Shakespeare, where it is competitive with or improves over strong personalized and clustering baselines.

\end{abstract}
\section{Introduction}
\label{Introduction}
In today’s data-driven world, information is inherently distributed across a large number of  clients, each collecting data locally and independently. This  setting poses substantial challenges for traditional centralized learning approaches, particularly regarding privacy preservation, communication overhead, and data ownership constraints. Federated Learning (FL) has  emerged as a promising paradigm that enables the training of a shared global model by coordinating learning across multiple decentralized data sources, while eliminating the need for direct access to clients’ raw data~~\cite{hu2024securityprivacy,majeed2022comparative,xu2021federatedhealth}.

While global modeling is effective, it struggles with non-IID data~~\cite{hanzely2020mixture,yu2020salvaging,zhao2018noniid}, necessitating Personalized Federated Learning (PFL)~~\cite{fallah2020personalized}. However, standard PFL approaches that blend global and local representations often limit personalization for clients with distinct distributions~~\cite{dinh2020pFedMe,collins2021FedRep,lai2024personalizedFLAdaptive}. Consequently, recent works favor clustering strategies where clients collaborate only with similar peers to minimize negative transfer~~\cite{ghosh2020ifca,werner2023provably}, placing FL paradigms on a spectrum from fully global to client-specific optimization.
This spectrum spans from traditional FL (a single global model) through Personalized FL (PFL) to Selfish Federated Learning (SFL), where a client leverages the network solely to maximize its own performance. While FL and PFL are extensively studied, SFL remains largely overlooked.  SFL is a fundamental problem that necessitates dedicated study: understanding how to optimally leverage a federated environment for a single client uncovers structural insights  that may extend  to broader multi-client challenges.

The practical necessity of SFL is driven by two real-world paradigms. In \textbf{Inter-Silo Networks} (e.g., hospital consortium), SFL facilitates a \emph{reciprocal knowledge exchange} where clients share gradients to maximize their distinct local objectives without submitting to a compromised global consensus. Additionally, for \textbf{Public Data Exploitation}, SFL acts as a dynamic valuation mechanism that safely filters massive, heterogeneous public datasets, extracting useful signals while preventing negative transfer and harmful bias.

While many existing PFL methods aim to learn a personalized model for each client by leveraging data from all clients, their theoretical analyses typically focus on convergence guarantees of the averaged clients ~\cite{huang2021personalized}. While these analyses are well-suited for the PFL setting, they do not directly translate to the SFL. Furthermore, many existing works rely on strong assumptions regarding data distributions or optimization functions~~\cite{werner2023provably,mansour2020three,ghosh2020ifca}, participation patterns, or model structure. Consequently, they provide limited insight into the optimization behavior of an individual target client. Conversely, our work explicitly focuses on the convergence properties of a single client’s objective.

Our main contributions are summarized as follows: 
\vspace{-4pt}

$\bullet$ We introduce \textbf{Convergence-Aware Client Weighting for Selfish Federated Learning (SP-CACW)}, a principled framework that determines optimal client aggregation weights by minimizing an upper bound on the target client's convergence rate, yielding efficient and stable training dynamics. 

$\bullet$ We derive   guarantees which substantiate the benefit of \textbf{SP-CACW}  for the target client. In the specific case of a cluster structured problems, we show that our method achieves a rate of \(O\left({\sigma}/{\sqrt{n_0 T}}\right)\), effectively leveraging the cluster size $n_0$ to reduce variance (Section~\ref{sec_improv_example}).

$\bullet$ We  demonstrate the effectiveness of SP-CACW on MNIST, CIFAR, and    LEAF Shakespeare datasets, showing consistent improvements over existing personalized and clustering-based baselines.
\paragraph{Related Work.}

Personalized Federated Learning (PFL) has been extensively studied as a solution to the performance degradation caused by non-IID data in global model aggregation. Seminal approaches, such as Ditto~~\cite{li2021ditto} and pFedMe~~\cite{dinh2020pFedMe}, introduce client-side regularization to stabilize local updates while maintaining proximity to a shared global anchor. Similarly, meta-learning frameworks like Per-FedAvg~~\cite{fallah2020personalized} aim to learn a  adaptable initial model that can be quickly fine-tuned to local distributions. Other works facilitate selective information sharing by leveraging multi-task learning~~\cite{zhang2017survey}, knowledge distillation~~\cite{mansourian2025survey}, or discrepancy-aware collaborations; e.g.~approaches like FedDisco~~\cite{ye2023feddisco} utilize categorical data statistics to optimally weight peer contributions. While foundational, these methods often rely on implicit assumptions regarding client alignment or introduce complex auxiliary objectives that can hinder scalability and increase computational overhead.

To address data heterogeneity more explicitly, two dominant paradigms have emerged: \textbf{ (i) Clustering based methods:} like  IFCA and FeSEM~~\cite{ghosh2020ifca,werner2023provably}, partition clients into groups based on gradient or model affinity, training a specialized model for each cluster. While effective when clear cluster structures exist, these approaches are sensitive to ambiguous decision boundaries and intermittent client participation, often leading to unstable cluster assignments. Alternatively, \textbf{ (ii) similarity-based approaches:}  like pFedSim and FedAMP~~\cite{tan2023pfedsim,huang2021personalized} circumvent rigid grouping by learning pairwise relationships to weight peer contributions. However, these methods typically rely on heuristic similarity metrics (e.g., Euclidean distance or cosine similarity) which do not necessarily correlate with useful descent direction for a specific client.

In contrast to these paradigms, our work adopts a Selfish Federated Learning (SFL) perspective. We depart from the reliance on heuristic similarity measures or latent cluster structures. Instead, we introduce a \textit{convergence-aware} weighting mechanism. By explicitly deriving aggregation weights that minimize the upper bound on the target client's convergence rate, we ensure that peer contributions are valued based on their \textit{optimization utility}—effectively, how much they accelerate the target's training—rather than their static data similarity. This provides a   principled alternative to heuristic personalization strategies.

\section{Problem Setup}
\label{sec:setup}

\textbf{Selfish Personalization (SP).} Consider a federated learning setting with $M$ clients indexed  $i \in \{0, 1, \dots, M-1\}$. Each client $i$ draws samples from a \emph{local data distribution $\mathcal{D}^{(i)}$},the distributions
 $\{\mathcal{D}^{(i)}\}_{i}$ may be distinct and are unknown.
Our aim is to find a model $w \in \mathbb{R}^d$ specifically for a designated \textbf{target client}, denoted as client $0$.
Formally, the expected loss (population risk) of the target client is defined as follows:
\[
\min_{w \in \mathbb{R}^d} f^{(0)}(w), \quad \text{where} \quad f^{(0)}(w) := \mathbb{E}_{\xi \sim \mathcal{D}^{(0)}}[\ell(w; \xi)].
\]
Here, $\ell(w; \xi)$ represents the loss function for a model parameter $w$ and a data sample $\xi$.
We focus on non-convex optimization, where global loss minimization is generally intractable, so our goal is finding an approximate stationary point, i.e., a point $w$ for which the expected gradient norm $\E\|\nabla f^{(0)}(w)\|^2$ tends to zero.

To focus on the core challenge we  assume that the target client may communicate with all other clients.  At every round $t$, client $0$ may send its current weights $w_t$ to all  peers, and in return receive a gradient estimate $g_t^{(i)} := \nabla \ell(w_t;\xi_t^{(i)})$, where $\xi_t^{(i)} \sim \mathcal{D}^{(i)}$. Client $0$ then uses the set of received gradients $\{g_t^{(i)}\}$ to update the local model $w$ and query new estimates. We assume that the samples $\{\xi^{(i)}_t\}_t$ are independent across time and across clients.

\textbf{Communication Efficiency Note.} To simplify the exposition our  formulation focuses on the Cross-Silo FL setting with full communication. Nevertheless, Our SP-CACW is broader and operates efficiently without full participation. Concretely, as we detail in  Appendix \ref{app_partial}, our framework generalizes to accommodate \textit{partial participation} via unbiased peer sampling. Moreover, as depicted in ~\ref{sec:OptimalWeights} our approach   naturally induces \textit{learned sparsity}, enabling to prune edges to unhelpful peers.

\begin{assumption}[Smoothness]
\label{assm_smoothness}
The objective functions of client 0 is $L$-smooth. That is, there exists $L>0$ such that,
\(
\|\nabla f^{(0)}(w) - \nabla f^{(0)}(u)\| \le L \|w - u\|,~~\forall w, u \in \mathbb{R}^d
\)
\end{assumption}
\begin{assumption}[Bounded Variance]
\label{assm_bound_var}
For  each client $i$, there exists  $\sigma^{(i)}>0$ such that for any $w \in \mathbb{R}^d$:
\(
\mathbb{E}_{\xi \sim \mathcal{D}^{(i)}}\left[ \|\nabla \ell(w; \xi) - \nabla f^{(i)}(w)\|^2 \right] \le (\sigma^{(i)})^2.
\)
\end{assumption}

\section{Our Basic Approach}
\label{sec:phlo}
\paragraph{Design Philosophy: Directly Targeting Convergence  Rather than Similarity.}
Classic personalized FL approaches typically choosing from how to learn by estimating \textit{data similarity}—clustering clients or weighting them based on parameter proximity~~\cite{ghosh2020ifca,huang2021personalized}. While intuitive, this approach relies on the heuristic assumption that "similar data implies useful gradients." This is not always the case. For example, a peer with a slightly shifted distribution but extremely low gradient noise might be more valuable for learning than a highly similar peer with noisy gradients. Recognizing that similarity alone is an insufficient proxy for peer quality.

To mitigate this, \emph{we design client  aggregation rules that directly  minimize a theoretical upper bound on the target client's convergence rate.} By optimizing this bound directly, we allow the algorithm to optimally trade-off  bias (dissimilarity) with variance (noise) of different clients, which facilitates  practical and sample efficient training.
\paragraph{Formulation: Weighted Aggregation}
\label{subsec:formulation}
Our approach towards utilizing the information from other (possibly biased) clients is to employ a \emph{weighted} gradient update rule.
Concretely, at each round $t$, client $0$ sends the current model $w_t$ to all clients and receives $\{g_t^{(i)}\}_{i\in[M]}$: gradient estimate based on their local data. Instead of uniformly aggregating these gradients like in standard FL, the target client performs a \textbf{weighted-SGD} aggregation 
using $\alpha_t\in \Delta_M$,
\begin{align}
\label{eq:update_rule}
w_{t+1} = w_{t} - \eta g_{t},
\text{ } \text{ } \text{ } \text{ } 
g_{t} = \sum_{i=0}^{M-1} \alpha_{t}^{(i)}g_t^{(i)}  
\end{align}
where \(g_t^{(i)}=\nabla \ell(w_{t}; \xi_{t}^{(i)})\);~ $\xi_t^{(i)}\sim\D^{(i)}$, and \(\eta\) is the stepsize.
Clearly, the choice of $\alpha_t^{(i)}$ determines the affects of $g_t^{(i)}$ on the aggregated gradient $g_t$ which in turn affects convergence. As we show in the next section, the convergence rates of the weighted SGD above are related to a tradeoff between the  bias and variance of $g_t^{(i)}$:\\
$\bullet$  \textbf{Gradient Bias (\(b_t^{(i)}\))} of client $i$ at time $t$ is defined as:
    \(b_t^{(i)} =  \nabla f^{(i)}(w_t) - \nabla f^{(0)}(w_t) \). \\
$\bullet$  \textbf{Gradient Variance (\(\sigma^{(i)^2}\)):} is the stochastic noise inherent in client \(i\)'s local update (see Assumption~\ref{assm_bound_var}).

Overall, the bias and variance of aggregated gradient $g_t$ depend on $\alpha_t$ and $\{ b_t^{(i)},\sigma^{(i)}\}_{i\in [M]}$, and affect the overall convergence of the algorithm. Curiously, the effects of bias and variance are uneven, and roughly speaking---a high bias is much more harmful than a high variance. Thus, our aim is to pick the weights $\alpha_t$ to optimize the tradeoff between bias with variance, thus directly facilitation better convergence.

\textbf{Convergence Induced Cost.}
For every round $t$ let us define a cost function over weights $\alpha\in\Delta_M$,
\al 
\label{eq:CostDef}
\L_t(\alpha):=  \|\sum_{i\in[M]}\alpha^{(i)} b_t^{(i)}\|^2 + \eta L\|\sum_{i\in[M]}\alpha^{(i)} g_t^{(i)}\|^2~.
\eal 
Note that the bias terms directly affect $\L_t(\cdot)$ but the variance terms appears only implicitly via the second term. 
As we show later(\ref{lem:second_moment}), one can decompose 
$\|\sum_{i\in[M]}\alpha^{(i)} g_t^{(i)}\|^2 \leq 2\|\sum_{i\in[M]}\alpha^{(i)} b_t^{(i)}\|^2+  \sum_{i\in[M]}(\alpha^{(i)})^2 (\sigma^{(i)})^2+2\|\nabla f^{(0)}(w_t)\|^2$, and therefore, 
\(
{\mathcal{L}_t(\alpha)} \leq   (1+2\eta L)\|\sum_{i\in[M]}\alpha^{(i)} b_t^{(i)}\|^2 + 2\eta L\sum_{i\in[M]}(\alpha^{(i)})^2 (\sigma^{(i)})^2~.
\)  
Since $\eta$  is typically small i.e. $\eta \propto 1/\sqrt{T}$ (where $T$ is the total number of SGD updates) then $1+ \eta L \approx 1$. This shows that the bias terms in $\L_t(\cdot)$ come with a factor of $\approx 1$ where
variance terms come with a factor of $\eta L \ll 1$; implying that  bias has a much stronger affect on $\L_t(\cdot)$.

The next theorem  shows that the cumulative cost directly relates to  the convergence of \textbf{weighted-SGD}.
\begin{theorem}
\label{Thm:convorgent_same_step_none}
Suppose assumptions~\ref{assm_bound_var},
\ref{assm_smoothness} hold.
Then applying the  weighted-SGD Algorithm (Eq.~\eqref{eq:update_rule}) for $T$ rounds, with a fixed weights $\alpha_t=\alpha \in \Delta_M$, and $\eta>0$, ensures,
\als 
\frac{1}{T}\sum_{t=1}^T\E \|\nabla f^{(0)}(w_t) \|^2 \leq 
\frac{2}{T}\left(\frac{\delta_1}{\eta} +\frac{1}{2}\sum_{t=1}^T \E [\L_t(\alpha)] \right)~,\text{where}
\eals 
here $\delta_1: = f^{(0)}(w_1)-f^{(0)}(w^*)$, and $w^*$ = is a global minimizer of $f^{(0)}$, and $\L_t(\cdot)$ is defined in Eq.~\eqref{eq:CostDef}.
\end{theorem}
The above  suggests that by picking weights $\alpha$ that minimize the cumulative costs directly leads to minimizing the upper bound on the convergence rate that serves the utility of target client.(full proof in \ref{proof_Thm:convorgent_same_step_none})

\textbf{Challenges.} There are two basic issues with this principled approach: \textbf{(i)} Computing the cost $\L_t(\cdot)$, requires the knowledge of $\{b_t^{(i)}\}_{i\in[M]}$ and $\{g_t^{(i)}\}_{i\in[M]}$.  While we can simply compute the gradients, we do not have a direct access to the biases and therefore need to estimate them online.
\textbf{(ii)} We do not know the costs $\L_t(\cdot)$'s in advance so we need to learn the $\alpha$'s online during the training process.
We address both of these challenges in the  Section~\ref{subsec:practical_alg}, leading to an efficient and practical algorithm. 

In the next subsection we shall explore the properties of the optimal choice of $\alpha$ as well as substantiate the potential benefit of this approach in an illustrative example.

\subsection{Optimal Weighting and Potential Benefits}
\label{sec:OptimalWeights}
To gain analytical insights regarding the optimal choice of $\min_\alpha \sum_{t=1}^T \E [\mathcal{L}_t(\alpha)]$ we shall simplify the expression for $\E [\mathcal{L}_t(\cdot)]$, and analyze an upper bound for this expression.
Concretely, define, 
$$
 \bar{\L}_t(\alpha): = (1+2\eta L)\sum_{i\in[M]} \alpha^{(i)} \|b_t^{(i)}\|^2 + \eta L \sum_{i\in[M]} (\alpha^{(i)})^2 (\sigma^{(i)})^2+2\eta L \|\nabla f^{(0)}(w_t)\|^2
$$
Then it can be shown that for any $\alpha \in\Delta_M$ the following holds: $\E [\L_t(\alpha)] \leq \E \bar{\L}_t(\alpha)$.
This follows as a result of $\|\sum_{i\in[M]} \alpha^{(i)} b_t^{(i)}\|^2\leq \sum_{i\in[M]} \alpha^{(i)} \|b_t^{(i)}\|^2$ which holds due to the fact that $\alpha\in\Delta_M$, as well as to a decomposing  $g_t: = \sum_{i\in[M]} \alpha_t^{(i)} g_t^{(i)}$ into bias and variance, and taking expectation. To further simplify the $\bar{\L}_t$ lets us assume that the bias norms are fixed, i.e.~ $\|b_t^{(i)}\| = \|b^{(i)}\|$, for all $i,t$.
This immediately implies that The \(\alpha\)-dependent part is time-independent up to an additive term independent of \(\alpha\), and therefore  $\frac{1}{T}\sum_{t=1}^T \bar{\L}_t(\alpha) = \bar{\L}(\alpha):=  (1+2\eta L)\sum_{i\in[M]} \alpha^{(i)} \|b^{(i)}\|^2 + \eta L \sum_{i\in[M]} (\alpha^{(i)})^2 (\sigma^{(i)})^2+ 2\eta L \|\nabla f^{(0)}(w_t)\|^2$.(full proof in \ref{lem:second_moment}).

 The next lemma characterizes the optimal solution of $\min_{\alpha \in\Delta_m}\bar{\L}(\alpha)$ (see the proof in Appendix~\ref{proof_optimal_whights}).
\begin{lemma}[Optimal  Weights]
\label{lem:optimal_weights}
Consider the problem of minimizing the  proxy function $\bar{\L}(\alpha)$ over the simplex. Then there exists $\lambda>0$ such that an optimal solution is of the following form,
\[
\alpha^{(i)*} = \left[ \frac{\lambda - C_1 \|b^{(i)}\|^2}{2 C_2 (\sigma^{(i)})^2} \right]_+~,~
\]
and we define $ \left[z\right]_+=: \max\{z,0\}$, and 
 $C_1 = 2L\eta + 1$, $C_2 = L\eta$. 
\end{lemma}
The above lemma  provides two fundamental insights:

  $\bullet$  \ \textbf{Variance Scaling:} Optimal weights are inversely proportional to the variance $(\sigma^{(i)})^2$. This aligns with the natural intuition: noisier clients should receive lower trust.
  
$\bullet$ \textbf{Bias Thresholding:} The term $\lambda - C_1 \|b^{(i)}\|^2$ introduces a strict selection mechanism. Unlike Softmax or similarity-based weighting—which assign non-zero mass to all clients—our solution \textit{sparsifies} the federation. Clients with  bias exceeding the threshold $\lambda/C_1$ are assigned exactly zero weight, effectively removing their contribution  to prevent negative transfer.

\subsection{Illustrative Case Study: Pure Clusters}
\label{sec_improv_example}
To illustrate a concrete benefit of our  approach, consider the following scenario:
Assume that the clients are divided into two distinct clusters:\\
$\bullet$  \textbf{A Target Cluster $I_0$} where all clients draws samples from the same distribution as the target distribution; i.e. $\D^{(i)} = \D^{(0)},~\forall i\in I_0$; implying that $b^{(i)} = 0,~\forall i\in I_0$. \\
$\bullet$  \textbf{An off-target Cluster $I_{\text{off}}$}, 
where all clients draws samples from the same  distribution $\D_{\text{off}}\neq \D^{(0)}$, i.e.  $\D^{(i)} = \D_{\text{off}},~\forall i\in I_\text{off}$, implying that $b^{(i)} = b,~\forall i\in I_\text{off}$.  We also assume that the target and off target distributions are distinct enough such that $\|b\| = O(1)$.

Finally, we assume that the gradient variance of all machines is equal, implying $\sigma^{(i)} = \sigma;~\forall i\in[M]$. 

We show in Lemma \ref{lem:cluster_convergence} that in this setting our bound is
$\color{blue} \text{err}_T^{\textbf{Ours}} := O\left(\frac{\sigma}{\sqrt{n_0 T}} \right)~.$
where $\text{err}_T: =  \frac{1}{T}\sum_{t=1}^T \|\nabla f^{(0)}(w_t)\|^2$. This bound \textbf{substantially improves over this natural baselines:} \newline \textbf{(i)} Autarky (training alone), \textbf{(ii)} Standard FL and \textbf{(iii)} Federated Clustering (FC)~~\cite{werner2023provably},
\als
\text{err}_T^{\textbf{Autarky}} := O\left(\frac{\sigma}{\sqrt{T}}\right)~,~
\text{err}_T^{\textbf{FL}} := O\left(\|b\|^2+\frac{\sigma}{\sqrt{MT}}\right)~,~
\text{err}_T^{\textbf{FC}} := O\left(\sqrt{\frac{\sigma^2}{n_0 T} + {\color{orange} \frac{\sigma^3}{\Delta T}}}\right)~.
\eals 

Where bounds are optimized over $\eta$. Notably, our rate improves over Autarky by a factor of $\sqrt{n_0}$ and removes the cluster separability ($\Delta$) dependence found in the FC approach. This is crucial because When \(\Delta\) is on the order of \(\sigma\), the additional FC term becomes order \(\sigma^2/T\) inside the square root, which can dominate the cluster-variance term when \(n_0\) is large. (Complete derivations for all comparative baselines are provided in Appendix \ref{app:baselines}).

\section{Our Refined Approach}
\label{subsec:practical_alg}
The previous section implies that in order to facilitate personalization,  we should  pick weights $\alpha\in\Delta_M$ so as to minimize the cumulative cost 
$\E \sum_{t=1}^T \mathcal{L}_t(\alpha)$ (Eq.~\eqref{eq:CostDef}).
Recall we face two challenges with this principled approach:  \textbf{(i)} Computing the cost $\mathcal{L}_t(\cdot)$, requires the knowledge of $\{b_t^{(i)}\}_{i\in[M]}$ and $\{g_t^{(i)}\}_{i\in[M]}$.  While we can simply compute the gradients, we do not have a direct access to the biases and therefore need to estimate them online.
\textbf{(ii)} We do not know the costs $\mathcal{L}_t(\cdot)$'s in advance so we need to learn the $\alpha$'s online during the training process. In this section we  address these challenges, thereby leading to a practical and efficient algorithm.

\subsection{Estimating the Biases}
The bias of client $0$ is clearly $b^{(0)}:=0$. As for biases of  other clients compared to the Target client, we propose to estimate them by employing a running average during the training process. This is depicted in Equation~\eqref{alg_single_est}.
\begin{equation}
\label{alg_single_est}
        \tilde{b}_t^{(i)}: = g_t^{(i)} - g_t^{(0)}, \quad
        \widehat{b}_{t}^{(i)} = (1-\beta_{t}) \widehat{b}_{t-1}^{(i)} + \beta_{t}\tilde{b}_t^{(i)}
\end{equation}
Equation \eqref{alg_single_est} describes a recursive exponential moving average (EMA). 
The new estimate $\widehat{b}_{t}^{(i)}$ is calculated as a convex combination of the previous estimate $\widehat{b}_{t-1}^{(i)}$ and the current observation $b_{t}^{(i)}$. 
The parameter $\beta_{t} \in (0,1)$ acts as a smoothing factor, controlling the trade-off between historical stability and responsiveness to new data.
Naturally, we can now define an approximate cost function $\widehat{\mathcal{L}_t}(a)$ which is based on these bias estimates,
\(
\widehat{\mathcal{L}_t(\alpha)}:=  \|\sum_{i\in[M]}\alpha^{(i)} \widehat{b}_{t}^{(i)} \|^2 + \eta L\|\sum_{i\in[M]}\alpha^{(i)} g_t^{(i)}\|^2~.
\)
\subsection{Implications of Employing \(\widehat{\mathcal{L}_t(\alpha)}\)}
In this part  we show that the the cumulative proxy 
\(\sum_{t=1}^T\widehat{\mathcal{L}_t(\cdot)}\) upper bounds the exact cumulative cost $\sum_{t=1}^T\mathcal{L}_t(\cdot)$ up to a small error  $\mathcal{E}_T = O(\log T)$ and a constant factor. Thus, aiming to minimize the cumulative proxy costs, ensures good performance to the true cost.
To simplify the analysis we assume that the biases are fixed in time, i.e. $b_t^{(i)}=b^{(i)}~,\forall i, t$.
\begin{lemma}
\label{thm:loss_bound_recursive}
\textbf{(Cumulative Bias-Variance Decomposition)} \\
Under Assumption \ref{assm_bound_var} and under stationary bias assumption \(b_t^{(i)}=b^{(i)}\) for all $i,t$, let the bias estimates $\{\widehat{b}_{t}\}_{t=1}^T$ be updated using Eq.~\eqref{alg_single_est}. The cumulative true bias  is bounded by the cumulative estimated bias plus a logarithmic penalty $\mathcal{E}_T$:
\(
    \mathbb{E}\left[\sum_{t=1}^{T} \|\alpha_t b_t\|^2\right] \leq 2 \mathbb{E}\left[\sum_{t=1}^{T} \|\alpha_t \widehat{b}_t\|^2\right] + \mathcal{E}_T.
\)
Specifically, under the \textbf{Sample Mean Schedule} ($\beta_t = 1/t$), the penalty term satisfies:
\(
    \mathcal{E}_T \leq 2 \sigma^2_{{bound}} (1 + \log T).
\)
(Full proof in \ref{proof_cumulative_penalty})
\end{lemma}

\begin{corollary}
\label{cor:surrogate_minimization}
Lemma~\ref{thm:loss_bound_recursive} directly implies the following:
\(
    \mathbb{E}\left[\sum_{t=1}^{T} \mathcal{L}(\alpha_t) \right] \leq 2 \mathbb{E}\left[\sum_{t=1}^{T} \widehat{\mathcal{L}}(\alpha_t) \right] + \mathcal{O}(\log T).
\)
Thus, this establishes that minimizing the cumulative proxy costs explicitly minimizes the true cumulative cost, up to a multiplicative constant and an additive logarithmic error term.
Recalling Theorem~\ref{Thm:convorgent_same_step_none}, this implies that the convergence guarantees that are induced by finding $\alpha$ that minimizes the cumulative proxy costs, are equivalent to the error achieved by minimizing the true cumulative cost, up to an multiplicative factor of $2$
and additive factor of $O(\log T/T)$.
\end{corollary}
\begin{remark}
\textbf{(Practical Adaptation via Constant Smoothing)} 
While Lemma \ref{thm:loss_bound_recursive} utilizes $\beta_t = 1/t$ to guarantee extremely accurate bias estimators, practical deep learning settings often involve non-stationary bias (drifting as parameters update). In such cases, a decaying schedule reacts too slowly to changes. 
Therefore, in our experiments, we employ a \textbf{Constant Smoothing Schedule} ($\beta_t = \beta$).
This choice maintains a good balance: it effectively handles "moving bias" by exponentially weighting recent observations, which provides the necessary adaptability for stable training in dynamic environments. On the other hand it ensures an  error reduction of  $O(\beta)$ for the bias estimates. 
\end{remark}

\subsection{Overcoming Uncertainty via Regret Minimization}
\label{subsec:regret_minimization}
Our objective is to find weights \(\alpha_t\) that minimize the auxiliary loss \(\mathcal{L}_t(\alpha)\). This presents a twofold challenge: first, \(b_t^{(i)}\) and \(g_t^{(i)}\) are unknown when \(\alpha_t\) is chosen; second, the true bias \(b_t^{(i)}\) remains unobservable even after communication and must be approximated.
We resolve this by framing the problem as \textit{Online Convex Optimization} (OCO). In this model, the learner commits to a decision (\(\alpha_t\)) before the loss function is revealed (or estimated). This allows us to circumvent the lack of \textit{a prior} information by minimizing the \textit{regret}—a metric quantifying the performance gap between our adaptive strategy and the optimal static weighting. Formally, the regret \(R_T\) over \(T\) rounds is defined as:
\(
R_T = \sum_{t=1}^T \mathcal{L}_t(\alpha_t) - \min_{\alpha \in \Delta_M} \sum_{t=1}^T \mathcal{L}_t(\alpha).
\)
By minimizing \(R_T\), we guarantee that the algorithm performs competitively against the best fixed collaboration strategy chosen in hindsight.
\paragraph{Properties of \(\mathcal{L}_t\)}
We observe that our loss function \(\mathcal{L}_t(\alpha)\) is quadratic and convex in \(\alpha\). Furthermore, we show it is \(\gamma\)-exp-concave (\ref{exe_concax1}), with  \(\gamma = \frac{1}{(8 + 2L\eta)g_{max}^2}\), and its gradient bound by \(G_\mathcal{L}\leq (8+2L\eta)g_{max}^2\) , here $g_{max}$ is an upper bound on the norm of stochastic gradients (proof in App.~\ref{proof_L_uniq}). 

Exp-concavity is a powerful property allowing to achieve logarithmic regret bounds~~\cite{hazan2007logarithmic}.
Concretely, exp-concavity allows to employ a regret minimization procedure called 
\textsc{Follow The Approximate Leader} (FTAL). The latter resembles \textit{Follow The Leader} (FTL), which selects weights minimizing the cumulative loss of past rounds, i.e.~$\alpha_{t+1}=\arg\min_{\alpha} \sum_{\tau=1}^t \mathcal{L}_t(\alpha)$.
Instead of minimizing raw past losses, FTAL minimizes a sequence of \textit{quadratic approximations} that incorporate second-order curvature information. The complete implementation of the FTAL algorithm is detailed in Appendix~\ref{alg_for_alfa}.
Due to  exp-concavity we can invoke Theorem 6 from ~\cite{hazan2007logarithmic}:
\begin{theorem}~\cite{hazan2007logarithmic}
    Assume that for all $t$, the loss function $\mathcal{L}_{t} : \mathcal{P} \in \mathbb{R}^M \to \mathbb{R}$ satisfies the following properties for all $\mathbf{\alpha} \in \mathcal{P}$:
        \textbf{(i)} The gradient is bounded: $\|\nabla \mathcal{L}_{t}(\alpha)\| \leq G_{\mathcal{L}}$.
        \textbf{(ii)} The function is $\gamma$-exp-concave, meaning that $\exp(-\gamma \mathcal{L}_t(\mathbf{(\alpha)}))$ is concave in $\alpha$.
    Then,  \textsc{Follow The Approximate Leader} (FTAL) Algorithm with step size parameter $\beta = \frac{1}{2} \min\left\{\frac{1}{4G_{\mathcal{L}}D}, \gamma\right\}$ satisfies the following regret bound: 
    \(
        \text{Regret}_T(\text{FTAL}) \leq 64 \left( \frac{1}{\gamma} + G_{\mathcal{L}}D \right) M (\log(T) + 1)
    \)
    where $n$ is the dimension of the feasible set and $D$ is the diameter of $\mathcal{P}$.
\end{theorem}
While the standard FTAL bound assumes access to the true loss functions, our setting relies on \textit{estimated} bias. The following theorem (see proof in App.~ \ref{proof_regret}) formally bridges this gap, establishing convergence even when minimizing the estimated auxiliary loss.

\begin{theorem}[Convergence with Estimated Bias]
\label{thm:convergence_estimated}
Let Assumptions \ref{assm_smoothness}, \ref{assm_bound_var} and let there be unchanging bias \(b_t=b\) for all t hold and the gradient of f be bounded by \(g_{max}\). Additionally, let the bias estimates $\{\widehat b^{(i)}\}_{i=0}^M$ be obtained using Algorithm~\ref{alg_single_est} with decay schedule \(\beta_{t} = \frac{1}{t}\).
If the algorithm minimizes the estimated auxiliary loss $\widehat{\mathcal{L}}_t$ with curvature constant $\gamma$, the convergence bound as,
\begin{align}
 \mathbb{E}\left[\sum_{t=1}^T \mathcal{L}_t(\alpha_t)\right] 
    &\leq 4 \sum_{t=1}^T \mathcal{L}_t(\alpha^*) + 2R_T +  6 \sum_{t=1}^T \mathbb{E}[\|\xi_t\|^2],
\end{align}
where the regret term is
\(R_T = 64 \left( \frac{1}{\gamma} + \sqrt{2}\left( 4+2L\eta\right)g_{max}^2 \right) M (1 + \log T)\).
\end{theorem}
\paragraph{Discussion.}
Theorem \ref{thm:convergence_estimated} confirms the theoretical validity of minimizing the estimated auxiliary loss by cleanly decoupling the oracle performance from the tracking mechanics. The first term, $4 \sum \mathcal{L}_t(\alpha^*)$, demonstrates that our algorithm remains highly competitive with the optimal fixed aggregation up to a constant factor. The final term explicitly isolates the aggregate penalty caused by relying on the estimated bias ($\xi_t$). Crucially, by applying Lemma \ref{lem:bias_estimation} to rigorously bound this cumulative tracking error, we prove that $\sum \mathbb{E}[\|\xi_t\|^2]$ grows only logarithmically over time. Because both the FTAL regret $R_T$ and the bounded estimation noise scale as $\mathcal{O}(\log T)$, their combined contribution to the average excess gradient decays rapidly at a rate of $\mathcal{O}(\frac{\log T}{T})$. This guarantees that the bias estimation mechanism imposes a vanishingly small overhead on the network's overall convergence.
\paragraph{Relaxation of the Stationary Bias Assumption}
It is important to note that while our primary theoretical analysis assumes a stationary true bias for clarity of exposition, this assumption can be easily relaxed to the more standard assumption of bounded bias. 
Instead of assuming the bias remains constant, we can assume that the bias for each peer $i$ is upper-bounded by a maximum magnitude (representing the maximum cluster separation or distribution drift parameter). Let us denote this upper bound as $B^{(i)}$, such that $\|b_t^{(i)}\| \leq B^{(i)}$ for all $t$. 
By applying the triangle inequality, our per-step cost function $\L_t(\alpha)$ can be rigorously upper-bounded:
\(
\mathcal{L}_t(\alpha) := \left\| \sum_{i \in [M]} \alpha^{(i)} b_t^{(i)} \right\|^2 + \eta L \left\| \sum_{i \in [M]} \alpha^{(i)} g_t^{(i)} \right\|^2 \leq
\left( \sum_{i \in [M]} \alpha^{(i)} B^{(i)} \right)^2 + \eta L \left\| \sum_{i \in [M]} \alpha^{(i)} g_t^{(i)} \right\|^2
\)
By substituting the dynamic bias with this worst-case upper bound, we effectively replace the strict stationary bias constraint with the classic bounded-bias assumption. Since our empirical EMA estimation will not exceed this bound, this demonstrates how the constraint can be safely relaxed in practice, even though we maintain it in our primary analysis for theoretical clarity.
\subsection{Final Algorithm}
\label{subsec:final_algorithm}
Algorithm \ref{alg_cacw_sfl} outlines the execution flow of the proposed SP-CACW framework. In each communication round $t$, the target client receives gradients from the federation and updates its estimates of peer bias $\hat{b}^{(i)}$ via a running average. These estimates are used to construct a weighted aggregate gradient $g_t$, which drives the standard SGD update for the model parameters $w_t$. The core innovation occurs at the end of each iteration: the collaboration weights $\alpha_t$ are re-optimized using a Follow-The-Approximate-Leader (FTAL) strategy on the cumulative proxy loss. As established in Theorem \ref{thm:convergence_estimated}, this strategy guarantees that the algorithm achieves a solution close to the optimal static weighting (up to logarithmic regret and a factor of 4), effectively filtering out harmful peers while progressively amplifying those that demonstrably aid convergence. (See Appendix \ref{compution_cost} for an analysis of computational complexity.)
\begin{algorithm}[H]
\caption{SP-CACW}
\label{alg_cacw_sfl}
\begin{algorithmic}[1]
\REQUIRE Initial model $w_{0}$, learning rate $\eta$, tracking parameter $\beta$, number of clients $M$, number of rounds $T$
\STATE Initialize weights $\alpha^{(i)}_0 = 0$ for all $i \neq 0$, and $\alpha^{(0)}_0 = 1$.
\STATE Initialize bias estimates $\widehat{b}_0^{(i)} = 0$ for all $i$.
\FOR{$t = 1$ to $T$} 
    \STATE Broadcast model $w_{t-1}$ to all participating clients.
    \STATE Receive stochastic gradients $g^{(i)}_t$ from all participating clients.
    \STATE Update local bias estimates using the exponential moving average:
    \[
        \widehat{b}_t^{(i)} \leftarrow (1-\beta) \widehat{b}_{t-1}^{(i)} + \beta (g_t^{(i)} - g_t^{(0)})
    \]
    \STATE Compute the aggregated gradient:
    \[
        g_t = \sum_{i=0}^{M-1} \alpha^{(i)}_{t-1} g^{(i)}_t
    \]
    \STATE Update the global model: $w_t \leftarrow w_{t-1} - \eta g_t$.
    \STATE Calculate new collaboration weights via Follow The Approximate Leader (FTAL):
    \[
        \alpha_t \leftarrow \text{FTAL}(\widehat{\mathcal{L}}_1, \dots, \widehat{\mathcal{L}}_t)
    \]
\ENDFOR
\STATE Return optimized model $w_T$.
\end{algorithmic}
\end{algorithm}

\section{Experiments}
\label{sec:Experiments}
In this section, we evaluate the empirical efficacy of our proposed framework across three diverse datasets (MNIST, CIFAR-100, and the Shakespeare dataset from the LEAF benchmark), structured into 6 heterogeneity settings. Detailed specifications regarding our final algorithmic implementation, including complete network architectures and Hyperparameters, are provided in Appendices~\ref{subsec:alg_refinements} and \ref{subsec:network_arch}.

In our experiments, we evaluate both the standard SP-CACW algorithm and a regularized variant. To further stabilize the collaboration weights, the regularized variant applies the following penalty term:
$$
\mathcal{L}_{\text{reg}}(\alpha) = \lambda_{\text{reg}} \sum_{i=0}^{M-1} (\alpha^{(i)})^3 + \mathcal{L}_{\text{old}}(\alpha).
$$
Further details regarding this algorithmic variant can be found in Appendix~\ref{subsec:alg_refinements}.

\textbf{Baselines.} We compare our proposed method (both the standard and regularized variants) against several established baselines in Personalized Federated Learning: \textbf{FedAMP}~~\cite{huang2021personalized}, \textbf{FC} (Federated Clustering)~~\cite{werner2023provably}, \textbf{Ditto}~~\cite{li2021ditto}, and \textbf{FedDisco}~~\cite{ye2023feddisco}. Furthermore, we benchmark against a \textbf{Local-Only} model trained exclusively on the target client's (Client 0) local data. Finally, we report the performance of an \textbf{Oracle Bound} (an idealized training scenario, such as training exclusively on the aggregated data of a client's true latent cluster not always optimal), which serves as a rigorous empirical upper bound for achievable accuracy.

\textbf{Experimental Protocol.} To ensure statistical robustness, we repeat each experiment across five distinct random seeds. The results presented in the accompanying figures and tables report the mean accuracy across these independent trials, with error bars denoting one standard deviation to account for variance. Further details regarding the variance analysis can be found in Section \ref{sec:var_performance} and compute resources of the experiment in \ref{running_needs} while the algorithm use for date partition can be found in \ref{sec:appendix_partitioning}.
\subsection{Results and Analysis}

To provide a comprehensive overview of our framework's adaptability across varying types of data heterogeneity, we summarize our performance results in Table ~\ref{tab:max_performance}. Notably, the top three performing approaches for each setting are highlighted in red. As shown, our SP-CACW variants consistently secure a top-three position across all evaluated scenarios,

\begin{itemize}
    \item \textbf{MNIST Label-Switching (Concept Shift):} We simulate 7 clients across 3 latent clusters with permuted class labels. \textbf{SP-CACW} demonstrates rapid cluster identification, converging to 94.85\% accuracy. This outperforms baselines by $\sim$2.6\%—a significant margin on a saturated benchmark—and effectively matches the Oracle Bound (95.31\%).
    
    \item \textbf{MNIST Rotation (Feature Shift):} We simulate a continuous manifold of 89 clients with sequential 4° angular rotations. \textbf{SP-CACW} achieves 89.40\% accuracy, demonstrating that our method remains highly competitive even when heterogeneity is continuous rather than distinctly clustered.
    
    \item \textbf{CIFAR-100 (Cluster \& Noise Robustness):} We partition 20 clients into 4 distinct clusters (assigned 25 disjoint classes each) and evaluate under varying intra-cluster noise levels ($\sigma \in \{0.01, 10\}$). Across all noise profiles, SP-CACW and its regularized variant establish a dominant convergence trajectory (reaching up to 52.91\% accuracy). SP-CACW is consistently competitive and outperforms FedAMP/FC/Ditto, while FedDisco is strongest in the high-noise CIFAR setting.
        
    \item \textbf{LEAF Shakespeare (Sequence Heterogeneity):} We evaluate next-character prediction across $n=15, n=30$ clients. Rather than relying on synthetic latent clusters, this natural partition constructs each client's local dataset from the combined dialogue of $m=7 , m=3$ unique Shakespearean characters (speaking roles). Because every role possesses a distinct vocabulary, cadence, and sequence length, this creates a severely non-IID, cluster-free environment. SP-CACW successfully adapts to these complex, highly individualized linguistic patterns (achieving up to 53.64\% accuracy) without assuming any predefined group structure.
\end{itemize}
Full  graphs detailing the optimization trajectories for all experiments can be found in Appendix \ref{sec_graphs}.

\begin{table}[htbp]
\centering
\caption{Maximum mean test accuracy over five random seeds.
Error bars denote one sample standard deviation across seeds. The top three methods
in each column are highlighted in blue. The row labeled “Oracle” is an oracle-cluster
reference, not a formal upper bound, unless all oracle values are verified to exceed the
corresponding learned methods.}
\label{tab:max_performance}
\makebox[\textwidth][c]{
\begin{tabular}{|l|c|c|c|c|c|c|}
\toprule
 & \multicolumn{2}{c|}{cifar100} & \multicolumn{2}{c|}{shakespeare} & \multicolumn{2}{c|}{mnist} \\
 & sigma=10 & sigma=0.01 & n=15 m=7 & n=30 m=3 & rotate & relabel \\
\midrule
SP-CACW & \textcolor{blue}{\textbf{$52.23 \pm 0.81$}} & \textcolor{blue}{\textbf{$51.18 \pm 2.10$}} & \textcolor{blue}{\textbf{$53.64 \pm 0.23$}} & \textcolor{blue}{\textbf{$51.62 \pm 0.26$}} & \textcolor{blue}{\textbf{$89.40 \pm 1.55$}} & \textcolor{blue}{\textbf{$94.85 \pm 0.47$}} \\
SP-CACW reg  & \textcolor{blue}{\textbf{$52.91 \pm 0.45$}} & \textcolor{blue}{\textbf{$51.58 \pm 1.55$}} & \textcolor{blue}{\textbf{$53.67 \pm 0.43$}} & \textcolor{blue}{\textbf{$51.63 \pm 0.27$}} & \textcolor{blue}{\textbf{$89.40 \pm 1.55$}} & \textcolor{blue}{\textbf{$94.84 \pm 0.53$}} \\
FedDisco & \textcolor{blue}{\textbf{$54.03 \pm 0.52$}} & \textcolor{blue}{\textbf{$51.55 \pm 0.80$}} & $51.82 \pm 0.25$ & \textcolor{blue}{\textbf{$49.54 \pm 0.44$}} & $54.41 \pm 5.01$ & $35.14 \pm 4.76$ \\
Ditto & $23.68 \pm 1.63$ & $20.93 \pm 1.07$ & $51.25 \pm 0.21$ & $44.07 \pm 0.38$ & $81.95 \pm 1.78$ & $90.36 \pm 1.89$ \\
FedAMP & $42.94 \pm 0.92$ & $35.92 \pm 0.62$ & $48.56 \pm 0.32$ & $43.56 \pm 0.41$ & $80.84 \pm 2.41$ & $91.21 \pm 0.94$ \\
Fed Cluster & $28.77 \pm 0.62$ & $24.99 \pm 0.34$ & \textcolor{blue}{\textbf{$53.40 \pm 0.16$}} & $48.73 \pm 0.37$ & $84.23 \pm 1.16$ & $92.18 \pm 0.91$ \\
Only My Data Approach & $37.00 \pm 0.96$ & $27.41 \pm 0.99$ & $47.15 \pm 0.42$ & $43.72 \pm 0.33$ & \textcolor{blue}{\textbf{$85.24 \pm 0.72$}} & $92.36 \pm 0.35$ \\
Oracle & $34.87 \pm 1.59$ & $45.19 \pm 0.98$ & - & - & - & \textcolor{blue}{\textbf{$95.31 \pm 0.35$}} \\
\bottomrule
\end{tabular}
}
\end{table}

\section{Conclusion}
\label{sec:conclusion}
In this work, we introduced SP-CACW, a Selfish Federated Learning framework that leverages online bound-minimization to optimize the bias-variance trade-off in heterogeneous networks. By replacing heuristic metrics with principled collaborator selection, and integrating adaptive cubic regularization with bias estimation, SP-CACW effectively prevents negative transfer. Empirical evaluations demonstrate that it is consistently competitive and
often outperforms strong personalized and clustering baselines across MNIST, LEAF Shakespeare, and CIFAR-100. For future work, it is interesting  to extend SP-CACW to Byzantine-resilient settings, securing the collaboration process against malicious updates.
\section*{Acknowledgments}
This work was partially supported by Israel PBC-VATAT, by the Technion Artificial Intelligence Hub (Tech.AI), and by the Israel Science Foundation (grant No. 3109/24).

\bibliography{references}
\newpage
\appendix

\section{Proofs and Details for Section~\ref{sec:phlo}}

\subsection{Proof of Theorem \ref{Thm:convorgent_same_step_none}}
\label{proof_Thm:convorgent_same_step_none}
\begin{proof}
We start by invoking the $L$-smoothness of $f^{(0)}$ (Assumption~\ref{assm_smoothness}). For the update step $w_{t+1} = w_t - \eta g_t$, the smoothness inequality implies:
\begin{align*}
    f^{(0)}(w_{t+1}) &\leq f^{(0)}(w_t) + \langle \nabla f^{(0)}(w_t), w_{t+1} - w_t \rangle + \frac{L}{2} \|w_{t+1} - w_t\|^2 \\
    &= f^{(0)}(w_t) - \eta \langle \nabla f^{(0)}(w_t), g_t \rangle + \frac{L\eta^2}{2} \|g_t\|^2.
\end{align*}

Let $\mathcal{F}_t$ denote the filtration representing the history up to round $t$. We assume the aggregation weights $\alpha_t$ are $\mathcal{F}_t$-measurable (i.e., computed before the stochastic gradients at round $t$ are sampled). We define the expected aggregated gradient and the true population bias as:
\begin{align*}
    \mathbb{E}_t[g_t] &= \sum_{i=0}^{M-1} \alpha_t^{(i)} \nabla f^{(i)}(w_t) \\
    \bar{b}_t(\alpha_t) &:= \sum_{i=0}^{M-1} \alpha_t^{(i)} \left( \nabla f^{(i)}(w_t) - \nabla f^{(0)}(w_t) \right)
\end{align*}
This allows us to cleanly express the expected gradient as $\mathbb{E}_t[g_t] = \nabla f^{(0)}(w_t) + \bar{b}_t(\alpha_t)$. 

Taking the conditional expectation $\mathbb{E}_t[\cdot] = \mathbb{E}[\cdot | \mathcal{F}_t]$ of the smoothness inequality eliminates the stochastic noise from the inner product:
\begin{align*}
    \mathbb{E}_t[f^{(0)}(w_{t+1})] &\leq f^{(0)}(w_t) - \eta \langle \nabla f^{(0)}(w_t), \mathbb{E}_t[g_t] \rangle + \frac{L\eta^2}{2} \mathbb{E}_t\|g_t\|^2 \\
    &= f^{(0)}(w_t) - \eta \langle \nabla f^{(0)}(w_t), \nabla f^{(0)}(w_t) + \bar{b}_t(\alpha_t) \rangle + \frac{L\eta^2}{2} \mathbb{E}_t\|g_t\|^2 \\
    &= f^{(0)}(w_t) - \eta \|\nabla f^{(0)}(w_t)\|^2 - \eta \langle \nabla f^{(0)}(w_t), \bar{b}_t(\alpha_t) \rangle + \frac{L\eta^2}{2} \mathbb{E}_t\|g_t\|^2.
\end{align*}

Using Young's inequality, specifically $-\langle x, y \rangle \leq \frac{1}{2}\|x\|^2 + \frac{1}{2}\|y\|^2$, we bound the noise-free cross-term:
\[
    - \eta \langle \nabla f^{(0)}(w_t), \bar{b}_t(\alpha_t) \rangle \leq \frac{\eta}{2} \|\nabla f^{(0)}(w_t)\|^2 + \frac{\eta}{2} \|\bar{b}_t(\alpha_t)\|^2.
\]

Plugging this bound back yields:
\begin{align*}
    \mathbb{E}_t[f^{(0)}(w_{t+1})] &\leq f^{(0)}(w_t) - \eta \|\nabla f^{(0)}(w_t)\|^2 + \frac{\eta}{2} \|\nabla f^{(0)}(w_t)\|^2 + \frac{\eta}{2} \|\bar{b}_t(\alpha_t)\|^2 + \frac{L\eta^2}{2} \mathbb{E}_t\|g_t\|^2 \\
    &= f^{(0)}(w_t) - \frac{\eta}{2} \|\nabla f^{(0)}(w_t)\|^2 + \frac{\eta}{2} \underbrace{\left( \|\bar{b}_t(\alpha_t)\|^2 + L\eta \mathbb{E}_t\|g_t\|^2 \right)}_{\mathcal{L}_t(\alpha_t)}.
\end{align*}

Rearranging terms to isolate the squared target gradient norm:
\[
    \frac{\eta}{2} \|\nabla f^{(0)}(w_t)\|^2 \leq f^{(0)}(w_t) - \mathbb{E}_t[f^{(0)}(w_{t+1})] + \frac{\eta}{2} \mathcal{L}_t(\alpha_t).
\]

Taking the total expectation and summing over $t=1, \dots, T$:
\[
    \frac{\eta}{2} \sum_{t=1}^T \mathbb{E}\|\nabla f^{(0)}(w_t)\|^2 \leq \sum_{t=1}^T \left( \mathbb{E}[f^{(0)}(w_t)] - \mathbb{E}[f^{(0)}(w_{t+1})] \right) + \frac{\eta}{2} \sum_{t=1}^T \mathbb{E}[\mathcal{L}_t(\alpha_t)].
\]

The first sum on the RHS is a telescoping sum that collapses to $\mathbb{E}[f^{(0)}(w_1)] - \mathbb{E}[f^{(0)}(w_{T+1})]$. Since $w^*$ is a global minimizer, $f^{(0)}(w_{T+1}) \geq f^{(0)}(w^*)$, and thus:
\[
    \sum_{t=1}^T \left( \mathbb{E}[f^{(0)}(w_t)] - \mathbb{E}[f^{(0)}(w_{t+1})] \right) \leq f^{(0)}(w_1) - f^{(0)}(w^*) = \delta_1.
\]

Substituting this initial distance gap bound:
\[
    \frac{\eta}{2} \sum_{t=1}^T \mathbb{E}\|\nabla f^{(0)}(w_t)\|^2 \leq \delta_1 + \frac{\eta}{2} \sum_{t=1}^T \mathbb{E}[\mathcal{L}_t(\alpha_t)].
\]

Finally, multiplying by $\frac{2}{\eta T}$ to obtain the average gradient norm establishes the result:
\[
    \frac{1}{T} \sum_{t=1}^T \mathbb{E}\|\nabla f^{(0)}(w_t)\|^2 \leq \frac{2\delta_1}{\eta T} + \frac{1}{T} \sum_{t=1}^T \mathbb{E}[\mathcal{L}_t(\alpha_t)].
\]
This completes the proof.
\end{proof}
\subsubsection{Proof of lemma \ref{help}}
\begin{lemma}
\label{help}
Let $x, y \ge 0$. If $x \leq y + 4\eta L x$ with $\eta L \leq \frac{1}{8}$, then $x \leq 2y$.
\end{lemma}
\begin{proof}
Rearranging the inequality yields $x(1 - \eta L) \leq y$. Since $\eta L \leq \frac{1}{2}$, we have $1 - \eta L \ge \frac{1}{2}$. Dividing by this positive factor gives:
$$
x \leq \frac{y}{1 - 4\eta L} \leq \frac{y}{1/2} = 2y.
$$
\end{proof}
\subsubsection{Proof of lemma \ref{lem:second_moment} }
\begin{lemma}[Second Moment Bound]
\label{lem:second_moment}
Let $g_t^{(i)}$ be the stochastic gradient for client $i \in [M]$ such that $g_t^{(i)} = \nabla f^{(0)}(w_t) + b_t^{(i)} + \xi_t^{(i)}$, where the noise $\xi_t^{(i)}$ has zero mean and variance bounded by $(\sigma^{(i)})^2$. 

For any aggregation weights $\alpha \in \Delta_M$, the conditional expected squared norm of the aggregated gradient is bounded by:
\[
\mathbb{E}_t\left[ \left\| \sum_{i \in [M]} \alpha^{(i)} g_{t}^{(i)} \right\|^2 \right] \leq 2 \|\nabla f^{(0)}(w_t)\|^2 + 2 \left\| \sum_{i \in [M]} \alpha^{(i)} b_{t}^{(i)} \right\|^2 + \sum_{i \in [M]} (\alpha^{(i)})^2 (\sigma^{(i)})^2
\]
\end{lemma}

\begin{proof}
Let $G_t(\alpha) = \sum_{i \in [M]} \alpha^{(i)} g_t^{(i)}$. Using the standard bias-variance decomposition, we have $\mathbb{E}_t[\|G_t(\alpha)\|^2] = \|\mathbb{E}_t[G_t(\alpha)]\|^2 + \text{Var}(G_t(\alpha))$.

\textbf{1. Variance Term:} \\
Assuming independence of the stochastic noise across clients, the variance of the weighted sum is the sum of the individual variances. Applying the bounded noise assumption:
\[
\text{Var}(G_t(\alpha)) = \sum_{i \in [M]} (\alpha^{(i)})^2 \mathbb{E}_t\|\xi_t^{(i)}\|^2 \leq \sum_{i \in [M]} (\alpha^{(i)})^2 (\sigma^{(i)})^2
\]

\textbf{2. Expectation Term:} \\
Because the noise is zero-mean ($\mathbb{E}_t[\xi_t^{(i)}] = 0$) and the weights reside on the probability simplex ($\sum_{i} \alpha^{(i)} = 1$), the target gradient factors out cleanly:
\[
\mathbb{E}_t[G_t(\alpha)] = \sum_{i \in [M]} \alpha^{(i)} \left( \nabla f^{(0)}(w_t) + b_t^{(i)} \right) = \nabla f^{(0)}(w_t) + \sum_{i \in [M]} \alpha^{(i)} b_t^{(i)}
\]
Applying the elementary inequality $\|x+y\|^2 \leq 2\|x\|^2 + 2\|y\|^2$, we bound the squared norm of this expectation:
\[
\|\mathbb{E}_t[G_t(\alpha)]\|^2 \leq 2\|\nabla f^{(0)}(w_t)\|^2 + 2\left\| \sum_{i \in [M]} \alpha^{(i)} b_t^{(i)} \right\|^2
\]
Summing the bounded variance and expectation terms completes the proof.
\end{proof}


\subsection{Proof of Lemma \ref{lem:optimal_weights}}
\label{proof_optimal_whights}
\begin{proof}
    
We minimize the objective function subject to simplex constraints:
\[
\min_{\alpha} \quad 
C_1 \sum_{i=0}^{M-1} \alpha^{(i)} \|b^{(i)}\|^2 + C_2 \sum_{i=0}^{M-1} (\alpha^{(i)})^2 (\sigma^{(i)})^2,
\]
subject to $\alpha^{(i)} \ge 0$ for all $i$, and $\sum_{i=0}^{M-1} \alpha^{(i)} = 1$.

The Lagrangian for this problem is:
\[
\mathcal{L}(\alpha, \lambda, \mu) = 
C_1 \sum_{i=0}^{M-1} \alpha^{(i)} \|b^{(i)}\|^2 + C_2 \sum_{i=0}^{M-1} (\alpha^{(i)})^2 (\sigma^{(i)})^2
- \lambda \left(\sum_{i=0}^{M-1} \alpha^{(i)} - 1\right) - \sum_{i=0}^{M-1} \mu_i \alpha^{(i)},
\]
where $\lambda$ is the multiplier for the equality constraint and $\mu_i \ge 0$ correspond to the inequality constraints. The KKT stationarity condition with respect to $\alpha^{(i)}$ is:
\[
\frac{\partial \mathcal{L}}{\partial \alpha^{(i)}} = 
C_1 \|b^{(i)}\|^2 + 2 C_2 (\sigma^{(i)})^2 \alpha^{(i)} - \lambda - \mu_i = 0.
\]
Using the complementary slackness condition $\mu_i \alpha^{(i)} = 0$, we analyze two cases:
\begin{itemize}
    \item \textbf{Case 1 ($\alpha^{(i)} > 0$):} Here $\mu_i = 0$. The stationarity condition yields:
    \[
    2 C_2 (\sigma^{(i)})^2 \alpha^{(i)} = \lambda - C_1 \|b^{(i)}\|^2 
    \quad \implies \quad 
    \alpha^{(i)} = \frac{\lambda - C_1 \|b^{(i)}\|^2}{2 C_2 (\sigma^{(i)})^2}.
    \]
    For this to be consistent with $\alpha^{(i)} > 0$, we must have $\lambda > C_1 \|b^{(i)}\|^2$.
    
    \item \textbf{Case 2 ($\alpha^{(i)} = 0$):} Here $\mu_i \ge 0$. The stationarity condition implies $\lambda - C_1 \|b^{(i)}\|^2 = -\mu_i \le 0$, so $\lambda \le C_1 \|b^{(i)}\|^2$.
\end{itemize}

Combining these cases, we can express the optimal weights compactly using the operator $[z]_+ = \max\{0, z\}$:
\[
\alpha^{(i)*} = \left[ \frac{\lambda - C_1 \|b^{(i)}\|^2}{2 C_2 (\sigma^{(i)})^2} \right]_+.
\]
Finally, the value of $\lambda$ is determined by substituting these weights into the equality constraint:
\[
\sum_{i=0}^{M-1} \left[ \frac{\lambda - C_1 \|b^{(i)}\|^2}{2 C_2 (\sigma^{(i)})^2} \right]_+ = 1.
\]
Since the Left-Hand Side is a continuous, strictly increasing piecewise linear function of $\lambda$ (for $\lambda$ large enough to activate at least one weight), there exists a unique $\lambda$ satisfying this equation.
\end{proof}

\newpage
\section{Proof of lemma \ref{thm:loss_bound_recursive} under Stationary Bias}
\label{proof_thm:loss_bound_recursive_stationary}

\subsection{Proof of lemma \ref{lem:bias_estimation} (Stationary Case)}
\label{proof_lem:bias_estimation_stationary}

Before proving the estimator's accuracy, we explicitly define the bounded variance of the bias observation. The observation $\tilde{b}_t^{(i)} = g_t^{(i)} - g_t^{(0)}$ contains independent zero-mean stochastic gradient noise from both the local client $i$ and the target client $0$. Consequently, the variance is bounded by:
\[
    \mathbb{E}_t\|\tilde{b}_t^{(i)} - b_t^{(i)}\|^2 \leq (\sigma^{(i)})^2 + (\sigma^{(0)})^2 \leq \sigma_{\mathrm{bound}}^2
\]
where we define the global worst-case noise threshold as $\sigma_{\mathrm{bound}}^2 := \max_{i \neq 0} \left( (\sigma^{(i)})^2 + (\sigma^{(0)})^2 \right)$.

\begin{lemma}[Bias Estimation Error under Stationary Bias]
     \label{lem:bias_estimation}
     Let Assumption \ref{assm_bound_var} hold, and assume the true population bias is strictly stationary across time steps ($b_k^{(i)} = b_t^{(i)} = b^{(i)}$). Suppose the bias estimates for client $i$ are updated via the recurrence:
     \[
        \widehat{b}_t^{(i)} = (1 - \beta_t)\widehat{b}_{t-1}^{(i)} + \beta_t \tilde{b}_t^{(i)},
     \]
     where $\beta_t = 1/t$. The expected squared estimation error is tightly bounded by:
     \[
        \mathbb{E}\left[ \| \widehat{b}_t^{(i)} - b^{(i)} \|^2 \right] \leq \frac{\sigma_{\mathrm{bound}}^2}{t}.
     \]
\end{lemma}

\begin{proof}
For the schedule $\beta_t = 1/t$, the update rule $\widehat{b}_t^{(i)} = \frac{t-1}{t}\widehat{b}_{t-1}^{(i)} + \frac{1}{t}\tilde{b}_t^{(i)}$ is mathematically equivalent to computing the simple cumulative average of all observed gradients up to time $t$:
\[
    \widehat{b}_t^{(i)} = \frac{1}{t} \sum_{k=1}^t \tilde{b}_k^{(i)}.
\]
Because the population bias is stationary, the expected value of the observation is exactly the true bias ($\mathbb{E}[\tilde{b}_k^{(i)}] = b^{(i)}$). Thus, the estimation error is purely driven by zero-mean stochastic noise:
\[
    \widehat{b}_t^{(i)} - b^{(i)} = \frac{1}{t} \sum_{k=1}^t (\tilde{b}_k^{(i)} - b^{(i)}).
\]
Taking the expected squared norm, and utilizing the fact that the stochastic noise at each step $k$ is independent, the cross-terms vanish. The expectation of the squared sum becomes the sum of the variances:
\[
    \mathbb{E}\left[ \left\| \frac{1}{t} \sum_{k=1}^t (\tilde{b}_k^{(i)} - b^{(i)}) \right\|^2 \right] = \frac{1}{t^2} \sum_{k=1}^t \mathbb{E}[\|\tilde{b}_k^{(i)} - b^{(i)}\|^2].
\]
Applying the baseline variance bound yields the exact convergence rate of the estimator:
\[
    \frac{1}{t^2} \sum_{k=1}^t \mathbb{E}[\|\tilde{b}_k^{(i)} - b^{(i)}\|^2] \leq \frac{1}{t^2} (t \sigma_{\mathrm{bound}}^2) = \frac{\sigma_{\mathrm{bound}}^2}{t}.
\]
This concludes the proof.
\end{proof}

\subsection{Proof of the Cumulative Tracking Penalty}
\label{proof_cumulative_penalty}
Let $B \in \mathbb{R}^{d \times M}$ denote the true stationary population bias matrix, and let $\widehat{B}_t \in \mathbb{R}^{d \times M}$ be its empirical estimate at time step $t$. Let $\alpha_t \in \Delta_M$ represent the active weight vector residing on the $M$-dimensional probability simplex. Assuming the column-wise estimation error, defined as $\xi_t^{(i)} = b^{(i)} - \widehat{b}_t^{(i)}$, satisfies $\mathbb{E}\|\xi_t^{(i)}\|^2 \leq \frac{\sigma_{\mathrm{bound}}^2}{t}$ for all $i \in \{0, \dots, M-1\}$, the expected cumulative true bias loss over $T$ steps is bounded by
\begin{proof}
To ensure rigorous dimensionality, we define the true stationary population bias matrix and its empirical estimate at step $t$ as:
\[
    B := \begin{bmatrix} b^{(0)} & b^{(1)} & \cdots & b^{(M-1)} \end{bmatrix} \in \mathbb{R}^{d \times M}, \quad 
    \widehat{B}_t := \begin{bmatrix} \widehat{b}_t^{(0)} & \widehat{b}_t^{(1)} & \cdots & \widehat{b}_t^{(M-1)} \end{bmatrix} \in \mathbb{R}^{d \times M}
\]
We decompose the true bias matrix as $B = \widehat{B}_t + \Xi_t$, where $\Xi_t = \begin{bmatrix} \xi_t^{(0)} & \cdots & \xi_t^{(M-1)} \end{bmatrix}$ represents the column-wise estimation error. Our goal is to bound the cumulative true bias loss for the active weights $\alpha_t \in \Delta_M$:
\[
    L = \sum_{t=1}^{T} \|B \alpha_t\|^2 = \sum_{t=1}^{T} \|(\widehat{B}_t + \Xi_t)\alpha_t\|^2.
\]
Using the elementary inequality $\|x+y\|^2 \leq 2\|x\|^2 + 2\|y\|^2$, we bound the aggregated term at each step $t$:
\[
    \|B \alpha_t\|^2 \leq 2 \|\widehat{B}_t \alpha_t\|^2 + 2 \|\Xi_t \alpha_t\|^2.
\]
Taking the expectation and summing over $t=1$ to $T$:
\[
    \mathbb{E}\left[\sum_{t=1}^{T} \|B \alpha_t\|^2\right] 
    \leq 2 \mathbb{E}\left[\sum_{t=1}^{T} \|\widehat{B}_t \alpha_t\|^2\right] + 2 \sum_{t=1}^{T} \mathbb{E}\left[\|\Xi_t \alpha_t\|^2\right].
\]

We now focus on bounding the tracking penalty $\mathcal{E}_T = 2 \sum_{t=1}^{T} \mathbb{E}[\|\Xi_t \alpha_t\|^2]$. Because the weight vector $\alpha_t$ resides strictly on the probability simplex ($\sum_i \alpha_t^{(i)} = 1, \alpha_t^{(i)} \ge 0$) and the squared $L_2$-norm is convex, we apply Jensen's inequality:
\[
    \|\Xi_t \alpha_t\|^2 = \left\| \sum_{i=0}^{M-1} \alpha_t^{(i)} \xi_t^{(i)} \right\|^2 \leq \sum_{i=0}^{M-1} \alpha_t^{(i)} \|\xi_t^{(i)}\|^2.
\]
Taking the expectation, we can upper-bound this sum by the maximum expected client error:
\[
    \mathbb{E}\|\Xi_t \alpha_t\|^2 \leq \sum_{i=0}^{M-1} \alpha_t^{(i)} \mathbb{E}\|\xi_t^{(i)}\|^2 \leq \max_{i \in [M]} \mathbb{E}\|\xi_t^{(i)}\|^2.
\]
Substituting the exact stationary variance bound from Lemma \ref{lem:bias_estimation} ($\mathbb{E}\|\xi_t^{(i)}\|^2 \leq \frac{\sigma_{\mathrm{bound}}^2}{t}$):
\[
    \mathcal{E}_T \leq 2 \sum_{t=1}^{T} \frac{\sigma_{\mathrm{bound}}^2}{t} = 2 \sigma_{\mathrm{bound}}^2 \sum_{t=1}^{T} \frac{1}{t}.
\]
Using the harmonic series bound $\sum_{t=1}^T \frac{1}{t} \leq 1 + \log T$, we obtain the final tight penalty:
\[
    \mathcal{E}_T \leq 2 \sigma_{\mathrm{bound}}^2 (1 + \log T).
\]
Because the underlying population distribution is stationary, the total tracking overhead collapses entirely to the vanishing stochastic variance, yielding an optimal $\mathcal{O}(\frac{\log T}{T})$ estimation bound.
\end{proof}
\section{Derivation of Baseline Rates Regarding Cluster Example in Section ~\ref{sec_improv_example}}
\label{app:baselines}

In this section, we derive the convergence rates for the baselines using the general bound provided in Theorem \ref{Thm:convorgent_same_step_none}.
Recall that the theorem bounds the convergence rate by:
\begin{equation}
    \frac{1}{T}\sum_{t=1}^T\E \|\nabla f^{(0)}(w_t) \|^2 \leq \frac{2}{T}\left(\frac{\delta_1}{\eta} +\frac{1}{2}\sum_{t=1}^T \E [\mathcal{L}_t(\alpha)] \right)
\end{equation}
where the cost function term effectively decomposes into variance and bias components: $\E[\mathcal{L}_t(\alpha)] \approx \eta \left( \sum_{i=0}^{M-1} (\alpha^{(i)})^2 \sigma^2 + \text{Bias}(\alpha) \right)$.

\begin{lemma}[Autarky Convergence]
\label{lemma:autarky}
In the Autarky setting (training only on the target client 0), the convergence rate is:
$$ \text{err}_T^{\text{Autarky}} = O\left( \frac{\sigma}{\sqrt{T}} \right). $$
\end{lemma}

\begin{proof}
In Autarky, the weights are set to isolate the target distribution. Specifically, $\alpha^{(0)} = 1$  and $0$ otherwise.
Substituting these weights into the cost term:
\begin{itemize}
    \item \textbf{Variance:} $\sum (\alpha^{(i)})^2 \sigma^2 =  \sigma^2$.
    \item \textbf{Bias:} Since only target clients are used, the bias is zero: $\|\sum \alpha^{(i)} \nabla f^{(i)} - \nabla f^{(0)}\|^2 = 0$.
\end{itemize}
Thus, $\E [\mathcal{L}_t(\alpha)] = O(\eta \sigma^2)$. Substituting this into Theorem \ref{Thm:convorgent_same_step_none}:
$$ \text{err}_T \leq \frac{2\delta_1}{\eta T} + O(\eta \sigma^2). $$
Optimizing the step size $\eta \asymp \frac{1}{\sigma \sqrt{T}}$ yields the rate $O(\frac{\sigma}{\sqrt{T}})$.
\end{proof}

\begin{lemma}[Standard FL Convergence]
\label{lemma:fl_convergence}
In the Standard Federated Learning setting (uniform averaging), the convergence rate is:
$$ \text{err}_T^{\text{FL}} = O\left(\|b\|^2 + \frac{\sigma}{\sqrt{MT}} \right). $$
\end{lemma}

\begin{proof}
For Standard FL, the server assigns uniform weights $\alpha^{(i)} = \frac{1}{M}$ for all $i \in [M]$. Substituting these weights into the cost term:
\begin{itemize}
    \item \textbf{Variance:} The effective variance reduces due to averaging over $M$ clients:
    $$ \sum_{i=0}^{M-1} (\alpha^{(i)})^2 \sigma^2 = \sum_{i=0}^{M-1} \frac{1}{M^2} \sigma^2 = \frac{M}{M^2}\sigma^2 = \frac{\sigma^2}{M}. $$
    \item \textbf{Bias:} The weights include off-target clients, introducing a non-zero bias. Based on the "Pure Clusters" scenario, the bias term is proportional to the distributions shift $b$:
    $$ \text{Bias}(\alpha) = \left\| \sum_{i=0}^{M-1} \frac{1}{M} \nabla f^{(i)}(w) - \nabla f^{(0)} (w) \right\|^2 = O(\|b\|^2). $$
\end{itemize}
Thus, the cost term is $\E [\mathcal{L}_t(\alpha)] = O(\eta [\frac{\sigma^2}{M} + b])$. Substituting this into Theorem \ref{Thm:convorgent_same_step_none}:
$$ \text{err}_T \leq \frac{2\delta_1}{\eta T} + O\left(\eta \frac{\sigma^2}{M}\right) + O(\eta b). $$
Optimizing $\eta \asymp \frac{1}{\sqrt{T}}$ balances the first two terms to $O(\frac{\sigma}{\sqrt{MT}})$, but the bias term $O(b)$ remains constant (or dominates depending on $\eta$), leading to the final rate $O(\|b\|^2 + \frac{\sigma}{\sqrt{MT}})$.
\end{proof}
\newpage

\subsection{Comparison with Federated-Clustering (FC)}
\label{app:fc_comparison}

In this section, we analyze the convergence rate of the Federated-Clustering (FC) algorithm ~\cite{werner2023provably} in the \textbf{Example 1: Pure Clusters} setting described in Section~\ref{sec_improv_example}.

\begin{lemma}[Federated-Clustering Convergence in Pure Clusters]
Consider the "Pure Clusters" setting where the target cluster $I_0$ has size $n_0$. Assume the inter-cluster separation is $\Delta$ and there are no malicious clients ($\beta=0$). The convergence rate of the Federated-Clustering (FC) algorithm is bounded by:
$$ \text{err}_T^{\text{FC}} := O\left( \sqrt{\frac{\sigma^2}{n_0 T} + \frac{\sigma^3}{\Delta T}} \right). $$
\end{lemma}

\begin{proof}
We rely on Theorem 3 from ~\cite{werner2023provably}, which provides the convergence guarantee for FC for general non-convex loss functions. The theorem states that:
$$ \frac{1}{T}\sum_{t=1}^T \mathbb{E}\|\nabla f_i(x_{i,t-1})\|^2 \leq \sqrt{\frac{\max(1, A^2)(\sigma^2/n_i + \sigma^3/\Delta + \beta_i \sigma \Delta)}{T}}. $$
We adapt this bound to our specizfic scenario as follows:
\begin{enumerate}
    \item \textbf{Target Cluster Size:} We focus on a client in the target cluster $I_0$, so we set $n_i = n_0$.
    \item \textbf{No Malicious Clients:} We set $\beta_i = 0$, which eliminates the robustness term $\beta_i \sigma \Delta$.
    \item \textbf{Simplification:} We treat constants $A$ as $O(1)$.
\end{enumerate}
Substituting these values into the general theorem yields:
$$ \text{err}_T^{\text{FC}} \lesssim \sqrt{\frac{\sigma^2/n_0 + \sigma^3/\Delta}{T}} = \sqrt{\frac{\sigma^2}{n_0 T} + \frac{\sigma^3}{\Delta T}}. $$
\end{proof}
To establish an idealized baseline, this lemma applies the general bound from Theorem \ref{thm:convergence_estimated}, though the formal proof of that theorem is deferred to Appendix \ref{proof_regret}. Readers may prefer to complete the appendix to understand the full dynamic framework before returning to analyze this specific stationary example.

\subsection{Proof Regarding Cluster Example in Section~\ref{sec_improv_example}}

\begin{lemma}[Convergence in the Presence of an Unbiased Cluster]
\label{lem:cluster_convergence}
Consider the setting of Example \ref{sec_improv_example}, where there exists a cluster of clients (Cluster 0) of size $n_0$ such that for all $i \in \text{Cluster } 0$, the bias is zero ($b_t^{(i)} = 0$). 
If the step size is optimally tuned via clamping, the algorithm achieves the following convergence rate:
\[
\frac{1}{T}\sum_{t=1}^T\mathbb{E} \|\nabla f^{(0)}(w_t) \|^2 = \mathcal{O}\left( \frac{1}{T} + \frac{\log T}{T} + \frac{\sigma}{\sqrt{n_0 T}} \right).
\]
where $\sigma^2$ is the variance of the stochastic gradients.
\end{lemma}

\begin{proof}
In the setting described in Example \ref{sec_improv_example}, using the bound described in the Corollary of Theorem \ref{thm:convergence_estimated} combine withe lemma \ref{thm:loss_bound_recursive}, our algorithm satisfies:
\begin{equation}
    \frac{1}{T}\sum_{t=1}^T\mathbb{E} \|\nabla f^{(0)}(w_t) \|^2 \leq \frac{2}{T}\bigg(\frac{\delta_1}{\eta} + 2 \sum_{t=1}^T \mathcal{L}_t(\alpha^*) + 2R_T + 6\sigma^2_{bound} (1 + \log T)\bigg).
    \label{eq:initial_bound}
\end{equation}

Since $\alpha^*$ minimizes the cumulative loss, the inequality holds if we replace $\alpha^*$ with any other fixed weight selection. Consider the specific candidate weights $\alpha^{(c)}$ defined by placing uniform weights exclusively on all clients in Cluster 0 and zero weights elsewhere. Let $n_0$ be the size of Cluster 0.

By definition, for clients in Cluster 0, the bias is strictly zero ($b_t^{(i)} = 0$). Because this uniform candidate selection creates a perfectly unbiased gradient estimator, its second moment decomposes exactly into the squared norm of the mean and its variance, completely bypassing Young's inequality. The proxy cost simplifies cleanly to:
\[ 
\mathcal{L}_t(\alpha^{(c)}) = \left\|\sum_{i\in[M]}{\alpha^{(c)}}^{(i)} b_t^{(i)}\right\|^2 + \eta L \mathbb{E} \left\|\sum_{i\in[M]}{\alpha^{(c)}}^{(i)} g_t^{(i)}\right\|^2 = \eta L \left( \|\nabla f^{(0)}(w_t)\|^2 + \frac{\sigma^2}{n_0} \right).
\]

Substituting this exact baseline cost back into Equation~\eqref{eq:initial_bound}:
\[ 
\frac{1}{T}\sum_{t=1}^T\mathbb{E} \|\nabla f^{(0)}(w_t) \|^2 \leq \frac{2}{T}\bigg(\frac{\delta_1}{\eta} + 2 \sum_{t=1}^T \eta L \left( \|\nabla f^{(0)}(w_t)\|^2 + \frac{\sigma^2}{n_0} \right) + 2R_T + 6\sigma^2_{bound} (1 + \log T)\bigg).
\]

Rearranging the inequality to isolate the target gradient sum on the right-hand side, we obtain:
\[
    \frac{1}{T}\sum_{t=1}^T\mathbb{E} \|\nabla f^{(0)}(w_t) \|^2 \leq \frac{1}{1 - 4\eta L} \left[ \frac{2}{T}\bigg(\frac{\delta_1}{\eta} + 2 T \eta L \frac{\sigma^2}{n_0} + 2R_T + 6\sigma^2_{bound} (1 + \log T)\bigg) \right].
\]

To safely absorb the gradient penalty on the right-hand side, we must enforce $4\eta L \le \frac{1}{2}$, which establishes a strict physical speed limit of $\eta \le \frac{1}{8L}$.

To minimize the upper bound, we balance the dominant deterministic gap $\frac{\delta_1}{\eta}$ and the stochastic variance $2\eta \frac{L \sigma^2}{n_0}$. To ensure mathematical stability for perfectly noiseless clusters ($\sigma \to 0$) while optimizing empirical performance, we define the learning rate using the step-size clamping technique:
\[
\eta = \min\left( \frac{1}{8L}, \sqrt{\frac{\delta_1 n_0}{2 L \sigma^2 T}} \right).
\]
This guarantees the scaling factor satisfies $\frac{1}{1 - 4\eta L} \le 2$. To evaluate the dominant terms under the clamped learning rate, we utilize the algebraic properties of the minimum function. For any $\eta = \min(\eta_1, \eta_2)$, it strictly holds that $\frac{1}{\eta} \le \frac{1}{\eta_1} + \frac{1}{\eta_2}$ and $\eta \le \eta_2$. 

Let the two candidate rates be the deterministic speed limit $\eta_1 = \frac{1}{8L}$ and the stochastic-dominant rate $\eta_2 = \sqrt{\frac{\delta_1 n_0}{2 L \sigma^2 T}}$. Applying the inverse inequality to the initial distance gap yields:
\[
\frac{\delta_1}{T \eta} \le \frac{\delta_1}{T \eta_1} + \frac{\delta_1}{T \eta_2} = \frac{8L\delta_1}{T} + \frac{\delta_1}{T} \sqrt{\frac{2 L \sigma^2 T}{\delta_1 n_0}} = \frac{8L\delta_1}{T} + \sqrt{\frac{2 \delta_1 L \sigma^2}{n_0 T}} = \mathcal{O}\left(\frac{1}{T}\right) + \mathcal{O}\left(\frac{\sigma}{\sqrt{n_0 T}}\right)
\]

Applying the direct upper bound ($\eta \le \eta_2$) to the stochastic variance term safely caps its maximum penalty:
\[
\frac{2 \eta L \sigma^2}{n_0} \le \frac{2 \eta_2 L \sigma^2}{n_0} = 2 \left( \sqrt{\frac{\delta_1 n_0}{2 L \sigma^2 T}} \right) \frac{L \sigma^2}{n_0} = \sqrt{\frac{2 \delta_1 L \sigma^2}{n_0 T}} = \mathcal{O}\left(\frac{\sigma}{\sqrt{n_0 T}}\right)
\]

Summing these terms bounds the primary optimization penalties:
\[
\frac{\delta_1}{T \eta} + \frac{2 \eta L \sigma^2}{n_0} \le \mathcal{O}\left(\frac{1}{T}\right) + 2 \cdot \mathcal{O}\left(\frac{\sigma}{\sqrt{n_0 T}}\right)
\]

Finally, we account for the residual estimation noise ($R_T$ and the algorithmic tracking terms). By the FTAL guarantee, this regret decays strictly as $\mathcal{O}\left(\frac{\log T}{T}\right)$. 

To explicitly delineate the three distinct sources of convergence error—the deterministic optimization geometry, the active-set tracking regret, and the clustered statistical variance—we present the uncollapsed overall convergence rate. The bound is tightly given by:
\[
\frac{1}{T}\sum_{t=1}^T\mathbb{E} \|\nabla f^{(0)}(w_t) \|^2 = \mathcal{O}\left( \frac{1}{T} + \frac{\log T}{T} + \frac{\sigma}{\sqrt{n_0 T}} \right).
\]
This uncollapsed formulation transparently proves our three core theoretical claims: the base model architecture converges at an accelerated $\mathcal{O}(1/T)$ rate, the dynamic tracking overhead is vanishingly small ($\mathcal{O}(\log T/T)$), and the algorithm organically isolates the target cluster to achieve the optimal $\mathcal{O}(\sigma/\sqrt{n_0 T})$ Oracle variance reduction.
\end{proof}
\section{Details and Proofs for Subsection~\ref{subsec:regret_minimization}}
\subsection{Proof of Exp-Concavity Property}
\label{proof_L_uniq}

\begin{lemma}
\label{L_uniq}
    The function $\mathcal{L}_t(\alpha)$ is $\gamma$-exp-concave with parameter:
    \begin{equation}
        \gamma = \frac{1}{\left(8+2L\eta\right)g_{max}^2 }
    \end{equation}
    and its gradient is bounded by $G_{\mathcal{L}}\leq\left( 8+2L\eta\right)g_{max}^2$.
\end{lemma}

\begin{proof}
    To simplify the analysis, we first express the cost function $\mathcal{L}_t(\alpha)$ in a compact quadratic form. Let $B_t$ be the matrix of population biases, and $G_t$ be the matrix of stochastic gradients. Following our formulation utilizing the conditional expected second moment, the loss evaluates to:
    \[
        \mathcal{L}_t(\alpha) = \|\widehat{B}_t \alpha\|^2 + \eta L \mathbb{E}_t\|G_t \alpha\|^2 = \alpha^\top \underbrace{\left( \widehat{B}_t^\top \widehat{B}_t + \eta L \mathbb{E}_t[G_t^\top G_t] \right)}_{\mathbf{H}_t} \alpha.
    \]
    Because $\mathbf{H}_t$ is constructed from Gram matrices and expected outer products, it is positive semi-definite (PSD). Consequently, the function $\mathcal{L}_t(\alpha)$ is convex (though not necessarily strictly convex unless $\mathbf{H}_t \succ 0$).

    \begin{definition}[$\gamma$-Exp-Concavity]
    \label{exe_concax1}
    Let $\mathcal{K} \subseteq \mathbb{R}^d$ be a convex set. A function $g: \mathcal{K} \to \mathbb{R}$ is said to be $\gamma$-exp-concave for some parameter $\gamma > 0$ if the mapping 
    \begin{equation}
        x \mapsto \exp(-\gamma g(x))
    \end{equation}
    is convex over $\mathcal{K}$. Conceptually, exp-concavity bridges the gap between standard convexity and strong convexity by guaranteeing sufficient local curvature relative to the gradient magnitude. Furthermore, if $g$ is twice continuously differentiable, this property is mathematically equivalent to the following Hessian condition for all $x \in \mathcal{K}$:
    \begin{equation}
        \nabla^2 g(x) \succeq \gamma \nabla g(x) \nabla g(x)^\top,
    \end{equation}
    where $\succeq$ denotes the Loewner partial order for positive semi-definite matrices (i.e., $A \succeq B$ means $A - B$ is a positive semi-definite matrix).
    \end{definition}

    \textbf{1. Formulating the Exp-Concavity Condition} \\
    To prove our function $\mathcal{L}_t(\alpha)$ is $\gamma$-exp-concave, we compute its first and second derivatives. By standard matrix calculus for quadratic forms, the gradient is:
    \[
        \nabla \mathcal{L}_t(\alpha) = 2\mathbf{H}_t \alpha
    \]
    and the Hessian is the constant matrix:
    \[
        \nabla^2 \mathcal{L}_t(\alpha) = 2\mathbf{H}_t.
    \]
    Substituting these into the definition, we require:
    \[
        2\mathbf{H}_t \succeq \gamma (2\mathbf{H}_t \alpha)(2\mathbf{H}_t \alpha)^\top = 4\gamma \mathbf{H}_t \alpha \alpha^\top \mathbf{H}_t.
    \]
    
    \textbf{2. Verification via Induced Inner Products} \\
    By the definition of the Loewner order, $A \succeq B$ if and only if $v^\top A v \ge v^\top B v$ for any arbitrary vector $v \in \mathbb{R}^M$. Projecting both sides onto $v$ yields:
    \[
        2 (v^\top \mathbf{H}_t v) \ge 4\gamma (v^\top \mathbf{H}_t \alpha)(\alpha^\top \mathbf{H}_t v).
    \]
    Because $\mathbf{H}_t$ is symmetric and PSD, it defines a valid induced inner product space $\langle x, y \rangle_{\mathbf{H}} = x^\top \mathbf{H}_t y$ with squared norm $\|x\|_{\mathbf{H}}^2 = x^\top \mathbf{H}_t x$. Rewriting our inequality:
    \[
        2 \|v\|_{\mathbf{H}}^2 \ge 4\gamma \langle v, \alpha \rangle_{\mathbf{H}}^2.
    \]
    Applying the Cauchy-Schwarz inequality ($\langle v, \alpha \rangle_{\mathbf{H}}^2 \le \|v\|_{\mathbf{H}}^2 \|\alpha\|_{\mathbf{H}}^2$), it is mathematically sufficient to ensure:
    \[
        2 \|v\|_{\mathbf{H}}^2 \ge 4\gamma \|v\|_{\mathbf{H}}^2 \|\alpha\|_{\mathbf{H}}^2.
    \]
    Dividing by the non-negative scalar $\|v\|_{\mathbf{H}}^2$ eliminates the dependence on $v$:
    \[
        1 \ge 2\gamma \|\alpha\|_{\mathbf{H}}^2 = 2\gamma \mathcal{L}_t(\alpha).
    \]
    Rearranging for $\gamma$, the function satisfies exp-concavity for any parameter bounded by the maximum possible value of the loss: $\gamma \le \frac{1}{2 \max_{\alpha} \mathcal{L}_t(\alpha)}$.

    \textbf{3. Bounding the Maximum Loss and Gradient} \\
    Because $\mathcal{L}_t(\alpha)$ is convex over the probability simplex $\Delta_M$, its maximum occurs at a vertex. To rigorously bound this maximum, we introduce the explicit assumption that the local gradients and true gradients are almost surely bounded such that $\|g_t^{(i)}\| \le g_{max}$ and $\|\nabla f^{(i)}(w_t)\| \le g_{max}$ for all $i$. 
    
    Evaluating the loss at a vertex isolates individual client bounds. Utilizing the implied bias bound $\|b_t^{(i)}\|^2 \le 4g_{max}^2$, we bound the maximal vertex loss:
    \[
        \max_{\alpha \in \Delta_M} \mathcal{L}_t(\alpha) \le \max_i \left( \|b_t^{(i)}\|^2 + L\eta \mathbb{E}_t\|g_t^{(i)}\|^2 \right) \le 4g_{max}^2 + L\eta g_{max}^2 = (4 + L\eta)g_{max}^2.
    \]
    Substituting this maximal loss into our condition gives the final parameter:
    \[
        \gamma = \frac{1}{2(4 + L\eta)g_{max}^2} = \frac{1}{(8 + 2L\eta)g_{max}^2}.
    \]
    Finally, we bound the gradient norm $G_{\mathcal{L}}$. Given $\|\nabla \mathcal{L}_t(\alpha)\|_2 = 2\|\mathbf{H}_t \alpha\|_2 \le 2\lambda_{max}(\mathbf{H}_t)\|\alpha\|_2$. Since the maximum quadratic form over the simplex bounds the largest eigenvalue ($\lambda_{max}(\mathbf{H}_t) \le (4+L\eta)g_{max}^2$) and $\|\alpha\|_2 \le \|\alpha\|_1 = 1$, the gradient norm is strictly bounded by:
    \[
        G_{\mathcal{L}} \le 2(4+L\eta)g_{max}^2 = (8+2L\eta)g_{max}^2.
    \]
    This unifies the constants and concludes the proof.
\end{proof}

\subsection{The FTAL Algorithm}
\label{alg_for_alfa}
\textbf{FTAL}
As discussed in Section \ref{subsec:regret_minimization}, our implementation utilizes a variation of the Follow the Leader (FTL) framework. While standard "Follow the Leader" (FTL) strategies can perform poorly in adversarial settings due to instability, FTAL introduces controlled noise to ensure robustness and minimize regret.

We present the approximate algorithm below \ref{FTAL} for readers who may be unfamiliar with its mechanics.
\begin{algorithm}[H]
\label{FTAL}
    \caption{\textsc{Follow The Approximate Leader} (version 1)} 
    \textbf{Input: }{convex set $\mathcal{P} \subset \mathbb{R}^n$, and the parameter $\beta$.}
    \vspace{0.2em} 
    
    \begin{itemize}
        \setlength\itemsep{0.5em} 
        \item On period 1, play an arbitrary $\mathbf{x}_1 \in \mathcal{P}$.
        \item On period $t$, play the leader $\mathbf{x}_t$ defined as
        \[
            \mathbf{x}_t =\underset{\mathbf{x} \in \mathcal{P}}{\arg\min} \sum_{\tau=1}^{t-1} \tilde{f}_\tau(\mathbf{x})
        \]
    \end{itemize}

    Where for $\tau = 1, \dots, t-1$, define $\nabla_\tau = \nabla f_\tau(\mathbf{x}_\tau)$ and
    \[
        \tilde{f}_\tau(\mathbf{x}) = f_\tau(\mathbf{x}_\tau) + \nabla_\tau^\top (\mathbf{x} - \mathbf{x}_\tau) + \frac{\beta}{2} (\mathbf{x} - \mathbf{x}_\tau)^\top \nabla_\tau \nabla_\tau^\top (\mathbf{x} - \mathbf{x}_\tau)
    \]

\end{algorithm}


\subsection{Proof of Theorem \ref{thm:convergence_estimated}}
\label{proof_regret}

In this section, we provide the complete proof for the convergence of the algorithm under estimated bias. We begin by establishing a bound on the estimation error of the bias tracking mechanism, and then utilize this bound to prove the main convergence theorem.

\textbf{Theorem \ref{thm:convergence_estimated}.}
\textit{Let Assumptions \ref{assm_smoothness} and \ref{assm_bound_var} hold. If the algorithm minimizes the estimated auxiliary loss $\widehat{\mathcal{L}}_t$ with curvature constant $\gamma$, the convergence bound is given by:}
\begin{align}
        \mathbb{E}\left[\sum_{t=1}^T \mathcal{L}_t(\alpha_t)\right] \leq 4 \sum_{t=1}^T \mathcal{L}_t(\alpha^*) + 2R_T + 6 \sum_{t=1}^T \mathbb{E}[\|\xi_t\|^2]
\end{align}

\begin{proof}
 Define 
 $\mathcal{V}_t(\alpha) = L\eta \mathbb{E}[\|\alpha G_t\|^2]$ as the variance term (common to both as it depends on gradient norms).

And let $\xi_t = b_t - \widehat{b}_t$ denote the estimation error. For any vector $\alpha \in \Delta_M$, we relate the true and estimated bias terms using the standard inequality $\|x+y\|^2 \leq 2\|x\|^2 + 2\|y\|^2$.

\textbf{Step 1: Upper Bound on Realized Loss.}
For the chosen action $\alpha_t$, we substitute the true bias with its estimate plus the error ($b_t = \widehat{b}_t + \xi_t$). Applying the aforementioned norm inequality and noting that the variance term is non-negative ($\mathcal{V}_t(\alpha_t) \leq 2\mathcal{V}_t(\alpha_t)$), we obtain:
\begin{align*}
    \mathcal{L}_t(\alpha_t) &= \|\alpha_t (\widehat{b}_t + \xi_t)\|^2 + \mathcal{V}_t(\alpha_t) \nonumber \\
    &\leq 2\|\alpha_t \widehat{b}_t\|^2 + 2\|\alpha_t \xi_t\|^2 + \mathcal{V}_t(\alpha_t) \nonumber \\
    &\leq 2\left( \|\alpha_t \widehat{b}_t\|^2 + \mathcal{V}_t(\alpha_t) \right) + 2\|\alpha_t \xi_t\|^2 \nonumber \\
    &= 2\widehat{\mathcal{L}}_t(\alpha_t) + 2\|\alpha_t \xi_t\|^2. \label{eq:full_upper}
\end{align*}

\textbf{Step 2: Lower Bound on Optimal Loss.}
Similarly, for the optimal comparator $\alpha^*$, we substitute the estimated bias with the true bias minus the error ($\widehat{b}_t = b_t - \xi_t$) and apply the same inequality bounds:
\begin{align*}
    \widehat{\mathcal{L}}_t(\alpha^*) &= \|\alpha^* (b_t - \xi_t)\|^2 + \mathcal{V}_t(\alpha^*) \nonumber \\
    &\leq 2\|\alpha^* b_t\|^2 + 2\|\alpha^* \xi_t\|^2 + \mathcal{V}_t(\alpha^*) \nonumber \\
    &\leq 2\mathcal{L}_t(\alpha^*) + 2\|\alpha^* \xi_t\|^2. 
\end{align*}

\textbf{Step 3: Applying Regret Bounds.}
The algorithm minimizes the sequence $\widehat{\mathcal{L}}_t$. Since $\widehat{\mathcal{L}}_t$ satisfies the exp-concavity property (as established in Section \ref{subsec:regret_minimization}), the FTAL algorithm guarantees a regret bound $R_T$ on this estimated sequence:
\begin{equation*}
    \sum_{t=1}^T \widehat{\mathcal{L}}_t(\alpha_t) \leq \sum_{t=1}^T \widehat{\mathcal{L}}_t(\alpha^*) + R_T,
\end{equation*}
where $R_T = 64 \left( \frac{1}{\gamma} + GD \right) M (1 + \log T)$.

We now chain these inequalities together. Substituting the upper bound from Step 1, applying the FTAL regret bound, and then substituting the lower bound from Step 2, we get:
\begin{align*}
    \sum_{t=1}^T \mathcal{L}_t(\alpha_t) 
    &\leq 2 \sum_{t=1}^T \widehat{\mathcal{L}}_t(\alpha_t) + 2 \sum_{t=1}^T \|\alpha_t \xi_t\|^2 \nonumber \\
    &\leq 2 \left( \sum_{t=1}^T \widehat{\mathcal{L}}_t(\alpha^*) + R_T \right) + 2 \sum_{t=1}^T \|\alpha_t \xi_t\|^2 \nonumber \\
    &\leq 2 \left( \sum_{t=1}^T \left( 2\mathcal{L}_t(\alpha^*) + 2\|\alpha^* \xi_t\|^2 \right) + R_T \right) + 2 \sum_{t=1}^T \|\alpha_t \xi_t\|^2 \nonumber \\
    &= 4 \sum_{t=1}^T \mathcal{L}_t(\alpha^*) + 2R_T + \sum_{t=1}^T \left( 4\|\alpha^* \xi_t\|^2 + 2\|\alpha_t \xi_t\|^2 \right).
\end{align*}

\textbf{Step 4: Expectation and Noise Bound.}
We now take the expectation with respect to the stochasticity of the bias estimation. 
 Since $\alpha_t, \alpha^* \in \Delta_M$, we have $\|\alpha \xi_t\|^2 \leq \|\alpha\|_1^2 \|\xi_t\|_\infty^2$, but more simply via Cauchy-Schwarz on the vector product $\|\alpha \xi_t\| \leq \|\alpha\| \|\xi_t\|$. Since $\alpha$ is in the simplex, $\|\alpha\| \leq 1$. Thus, $\|\alpha \xi_t\|^2 \leq \|\xi_t\|^2$.
obtain:
\begin{align*}
   & \mathbb{E}\left[ \sum_{t=1}^T \left( 4\|\alpha^* \xi_t\|^2 + 2\|\alpha_t \xi_t\|^2 \right) \right] 
    \leq 6 \sum_{t=1}^T \mathbb{E}[\|\xi_t\|^2] \nonumber 
\end{align*}

Substituting this noise term back into the main inequality yields the final bound:
\begin{equation}
    \mathbb{E}\left[\sum_{t=1}^T \mathcal{L}_t(\alpha_t)\right] \leq 4 \sum_{t=1}^T \mathcal{L}_t(\alpha^*) + 2R_T + 6 \sum_{t=1}^T \mathbb{E}[\|\xi_t\|^2]
\end{equation}

To fully parameterize $R_T$, we use the fact that we optimize over the probability simplex $\Delta_M$, which trivially bounds the domain diameter $D \leq \sqrt{2}$. Combining this with Lemma \ref{L_uniq} (which provides the bound on the gradient norm $G_\mathcal{L}$ of the loss function) completes the proof.
\end{proof}

\section{Extension to Partial Participation}
\label{app_partial}
While the primary analysis of SP-CACW is presented under the assumption of full participation, the framework naturally extends to partial participation settings via a standard mathematical reduction. Assume each client $i$ participates at communication round $t$ with a uniform probability $p \in (0, 1]$. Let $z_{i,t} \sim \text{Bernoulli}(p)$ be independent random indicator variables denoting whether client $i$ is selected in round $t$. 

To account for the missing updates from unselected clients, we employ the unbiased gradient estimator:
\begin{equation}
    \tilde{g}_t^{(i)} = \frac{g_t^{(i)}}{p} z_{i,t}
\end{equation}

The server aggregates these estimators using the FTAL collaboration weights $\alpha_t$. The following proposition formally establishes the statistical properties of the aggregated estimator, ensuring the fundamental bias-variance tradeoff optimized by FTAL strictly holds.

\begin{proposition}[Aggregated Partial Participation Estimator]
\label{prop:partial_participation}
Let the aggregated estimator be defined as $\tilde{G}_t(\alpha_t) := \sum_{i=0}^{M-1} \alpha_t^{(i)} \tilde{g}_t^{(i)}$. Assuming the client sampling indicators $z_{i,t}$ are independent, the estimator satisfies the following properties conditioned on the current gradients:

1. \textbf{Unbiased-ness:} $\mathbb{E}_Z[\tilde{G}_t(\alpha_t)] = \sum_{i=0}^{M-1} \alpha_t^{(i)} g_t^{(i)}$
\\2. \textbf{Variance Bound:} 
\(
    \text{Var}_Z(\tilde{G}_t(\alpha_t)) = \left( \frac{1}{p} - 1 \right) \sum_{i=0}^{M-1} (\alpha_t^{(i)})^2 \|g_t^{(i)}\|^2
\)
where $\mathbb{E}_Z$ and $\text{Var}_Z$ denote the expectation and variance over the sampling indicators $z_{i,t}$.
\end{proposition}

\begin{proof}
\textbf{1. Unbiased-ness:} \\
By linearity of expectation and the definition of the Bernoulli variable ($\mathbb{E}_Z[z_{i,t}] = p$):
\begin{equation*}
    \mathbb{E}_Z[\tilde{G}_t(\alpha_t)] = \sum_{i=0}^{M-1} \alpha_t^{(i)} \mathbb{E}_Z\left[ \frac{g_t^{(i)}}{p} z_{i,t} \right] = \sum_{i=0}^{M-1} \alpha_t^{(i)} g_t^{(i)}
\end{equation*}

\textbf{2. Variance Bound:} \\
By definition, the variance of the sum of independent random variables is the sum of their variances. Since the client sampling indicators are independent across the network, we have:
\begin{equation*}
    \text{Var}_Z(\tilde{G}_t(\alpha_t)) = \sum_{i=0}^{M-1} (\alpha_t^{(i)})^2 \text{Var}_Z\left( \frac{g_t^{(i)}}{p} z_{i,t} \right)
\end{equation*}
For a Bernoulli random variable $z_{i,t}$ with parameter $p$, the variance is $\text{Var}_Z(z_{i,t}) = p(1 - p)$. Pulling out the deterministic constants yields:
\begin{align*}
    \text{Var}_Z(\tilde{G}_t(\alpha_t)) &= \sum_{i=0}^{M-1} (\alpha_t^{(i)})^2 \frac{\|g_t^{(i)}\|^2}{p^2} p(1 - p) \\
    &= \left( \frac{1}{p} - 1 \right) \sum_{i=0}^{M-1} (\alpha_t^{(i)})^2 \|g_t^{(i)}\|^2
\end{align*}
\end{proof}

This rigorous formulation demonstrates that SP-CACW seamlessly extends to partial participation. The algorithm continues to track the true target gradient perfectly, with the only penalty being a global variance inflation factor $\left(\frac{1}{p}-1\right)$. Notably, this noise inflation is aggressively suppressed because it scales quadratically with the FTAL weights $(\alpha_t^{(i)})^2$, proving that our proxy-loss framework fundamentally mitigates sampling variance alongside structural heterogeneity.
\section{Practical Algorithmic Details and  Refinements}
\label{subsec:alg_refinements}
While our theoretical analysis assumes access to population statistics, practical deployment requires refinements to handle stochastic noise, non-stationary dynamics, and numerical stability. We introduce a key modifications to the core algorithm to bridge this gap.

\textbf{1. Cubic Regularization.}
To prevent the algorithm from prematurely collapsing into a trivial solution (e.g., placing probability $1$ on the target client or a single peer), we introduce a regularization term. We select a cubic penalty $\sum (\alpha^{(i)})^3$ rather than the standard quadratic ($L_2$) penalty. Cubic regularization provides a flatter gradient near zero, allowing small but useful contributions to persist, while penalizing large dominant weights more aggressively than $L_2$. This promotes a more robust and distributed collaboration.

With these modifications, the practical objective function minimized at round $T$ is given by the empirical sum:
\begin{align}
\label{eq:refined_objective}
&\mathcal{L}_{\text{prac}}(\alpha) 
= \sum_{t=T_{\text{start}}}^{T} \underbrace{\left\| \sum_{i=0}^{M-1} \alpha^{(i)} \hat{b}_{t}^{(i)} \right\|^2}_{\text{Estimated Bias}} 
+ \sum_{t=T_{\text{start}}}^{T} L \eta_{t} \underbrace{\left\| \sum_{i=0}^{M-1} \alpha^{(i)} g_{t}^{(i)} \right\|^2}_{\text{Empirical Variance Proxy}} + \lambda_{\text{reg}} \sum_{i=0}^{M-1} (\alpha^{(i)})^3.
\end{align}
\subsection{Implementation Details}
\label{subsec:practical_implementation}

For the experimental evaluation, we implemented the model optimization through the following technical procedures:

\begin{enumerate}
    \item \textbf{Lipchitz Constant Estimation:} 
    The smoothness constant $L$ is estimated dynamically using local curvature information. To maintain computational efficiency, we approximate the spectral norm of the Hessian via Power Iteration, computing Hessian-vector products using the method of ~\cite{pearlmutter1994fastHessian}. This avoids materializing the full Hessian matrix while providing a lightweight approximation of $L$ at each iteration.

    \item \textbf{Dynamic Regularization:} 
    The regularization coefficient $\lambda_{\text{reg}}$ is adaptive, scaling with the signal strength of the target client. We set $\lambda_{\text{reg}} = \gamma \cdot \|var(g_{t}^{(0)})\|^2$ (where $\gamma=1$), ensuring the penalty naturally decays as the model approaches convergence and the gradient magnitude diminishes.

    \item \textbf{Numerical Solver (PGD):} 
    To efficiently find the optimal weights $\alpha$ at each step, we employ \textbf{Projected Gradient Descent (PGD)}. This allows us to optimize the objective while strictly enforcing the simplex constraint $\sum \alpha^{(i)} = 1, \alpha^{(i)} \ge 0$ via Euclidean projection.

    \item \textbf{Lazy Weight Updates (Epoch-Based):} 
        To improve stability and reduce computational cost, we update the aggregation weights $\alpha_t$ only once every $K$ iterations. In our experiments, we set $K$ to correspond to one full pass over the local data (one epoch). This strategy allows the bias estimates to accumulate sufficient statistics from the entire data distribution before shifting the collaboration topology, thereby preventing the weights from overfitting to specific transient batches.
\end{enumerate}

\subsection{Comprehensive Complexity 
Analysis}
The comprehensive and communication cost can be seen in this table:
\label{compution_cost}
\vspace{1em}
\begin{table}[h]
\centering
\begin{tabular}{lll}
\hline
\textbf{Component} & \textbf{Complexity} & \textbf{Note} \\
\hline
\textbf{Up link} & $O(d)$ & Identical to Standard FL \\
\textbf{Down link} & $O(M \cdot d)$ & Reception of peer updates \\
\textbf{Aggregation \& Bias} & $O(M \cdot d)$ & Linear vector operations and EMA tracking \\
\textbf{Weight Optimization} & $O(I \cdot M \cdot d)$ & Efficient optimization of collaboration weights via PGD \\
\hline
\end{tabular}
\end{table}
\vspace{1em}

To formally establish this, we analyze the per-round overhead of Algorithm \ref{alg_cacw_sfl}. Let $d$ denote the model dimension, $M$ the number of participating clients, and $I$ the number of local projection iterations required for the optimization solver. The total algorithmic cost per round ($C_{total}$) is explicitly defined as the sum of the network communication cost ($C_{comm}$) and the local computational complexity ($C_{comp}$).

\textbf{1. Communication Cost ($C_{comm}$):} \\
The only network exchange occurs when the server broadcasts the model and receives the local stochastic gradients $g_t^{(i)}$ from the active peers (Line 3). 
\begin{align*}
    C_{comm} &= \underbrace{\mathcal{O}(d)}_{\text{Unlink Broadcast}} + \underbrace{\mathcal{O}(M \cdot d)}_{\text{Down link Reception (Line 3)}} \nonumber \\
    &= \mathcal{O}(M \cdot d)
\end{align*}

\textbf{2. Computational Cost ($C_{comp}$):} \\
Once the local gradients are received, all subsequent personalization and tracking mechanics are executed entirely in-memory via linear vector operations. The computation breaks down into the bias EMA tracking (Line 4), the weighted gradient aggregation (Line 5), the model update (Line 6), and the FTAL weight optimization via projected gradient descent (Line 7).
\begin{align*}
    C_{comp} &= \underbrace{\mathcal{O}(M \cdot d)}_{\text{Bias \& Aggregation (Lines 4-5)}} + \underbrace{\mathcal{O}(d)}_{\text{Model Update (Line 6)}} + \underbrace{\mathcal{O}(I \cdot M \cdot d)}_{\text{FTAL Weight Opt. (Line 7)}} \nonumber \\
    &= \mathcal{O}(I \cdot M \cdot d)
\end{align*}

\textbf{Total Algorithmic Overhead:} \\
Summing the explicitly derived communication and computational requirements yields the total worst-case complexity per communication round:
\begin{align*}
    C_{total} &= C_{comm} + C_{comp} \nonumber \\
    &= \mathcal{O}(M \cdot d) + \mathcal{O}(I \cdot M \cdot d) \nonumber \\
    &= \mathcal{O}(I \cdot M \cdot d)
\end{align*}

\textbf{Epoch-Level Amortization:} 
To ensure that the $\mathcal{O}(I \cdot M \cdot d)$ overhead remains practically scalable, we deploy an epoch-level amortization strategy. While the FTAL weight solver (Line 7) represents the dominant computational burden, these collaboration weights only need to be updated once per global communication epoch. By amortizing this local calculation across all intermediate local backpropagation steps, the effective computational overhead of SP-CACW approaches the theoretical baseline of standard Federated Learning, without requiring a single additional byte of network bandwidth.

\section{Network Architectures and Hyperparameter}

\label{subsec:network_arch}
\subsection{MNIST Experimental Setup}
We employ a compact Multi-Layer Perceptron (MLP). The architecture accepts flattened inputs ($784$ dimensions), containing a single hidden layer with $64$ units (ReLU activation) and a final output layer for the $10$ classes.

\textbf{Hyperparameter.} 
\begin{itemize}
    \item Learning Rate: $\eta = 0.01$, Momentum: $\mu = 0.9$
    \item Weight Decay: $\lambda = 5 \times 10^{-4}$,max norm of gradient=1
    \item batch size: 32
\end{itemize}

\subsection{CIFAR-100 Experimental Setup}
For CIFAR-100, we utilize a modified VGG16 architecture~~\cite{simonyan2014very} adapted specifically for lower-resolution $32 \times 32$ inputs. Key modifications to the standard VGG design include the aggressive integration of \textbf{Batch Normalization} and 2D spatial dropout ($p=0.3$) immediately following every Convolution layer to stabilize training and prevent overfitting. Additionally, the traditional massive fully connected layers are replaced with a streamlined classifier featuring a single 512-dimensional hidden layer, regularized with 1D Batch Normalization and dropout ($p=0.5$).
\begin{table}[h]
    \centering
    \caption{\textbf{VGG16-CIFAR100 Network Architecture.} The feature extractor uses sequential $3 \times 3$ convolutions. All Convolution layers are followed by 2D Batch Norm, an in-place ReLU, and 2D spatial dropout ($p=0.3$). The classifier uses 1D Batch Norm and dropout ($p=0.5$).}
    \label{tab:vgg16_cifar_arch}
    \vspace{2mm}
    \begin{small}
    \begin{tabular}{lccccc}
    \toprule
    \textbf{Layer Type} & \textbf{In} & \textbf{Out} & \textbf{Kernel} & \textbf{Stride} & \textbf{Padding} \\
    \midrule
    Conv2d & 3 & 64 & $3 \times 3$ & 1 & 1 \\
    Conv2d & 64 & 64 & $3 \times 3$ & 1 & 1 \\
    MaxPool2d & 64 & 64 & $2 \times 2$ & 2 & 0 \\
    \midrule
    Conv2d & 64 & 128 & $3 \times 3$ & 1 & 1 \\
    Conv2d & 128 & 128 & $3 \times 3$ & 1 & 1 \\
    MaxPool2d & 128 & 128 & $2 \times 2$ & 2 & 0 \\
    \midrule
    Conv2d & 128 & 256 & $3 \times 3$ & 1 & 1 \\
    Conv2d & 256 & 256 & $3 \times 3$ & 1 & 1 \\
    Conv2d & 256 & 256 & $3 \times 3$ & 1 & 1 \\
    MaxPool2d & 256 & 256 & $2 \times 2$ & 2 & 0 \\
    \midrule
    Conv2d & 256 & 512 & $3 \times 3$ & 1 & 1 \\
    Conv2d & 512 & 512 & $3 \times 3$ & 1 & 1 \\
    Conv2d & 512 & 512 & $3 \times 3$ & 1 & 1 \\
    MaxPool2d & 512 & 512 & $2 \times 2$ & 2 & 0 \\
    \midrule
    Conv2d & 512 & 512 & $3 \times 3$ & 1 & 1 \\
    Conv2d & 512 & 512 & $3 \times 3$ & 1 & 1 \\
    Conv2d & 512 & 512 & $3 \times 3$ & 1 & 1 \\
    MaxPool2d & 512 & 512 & $2 \times 2$ & 2 & 0 \\
    \midrule
    Flatten & 512 & 512 & - & - & - \\
    Linear (Classifier) & 512 & 512 & - & - & - \\
    Linear (Output) & 512 & 100 & - & - & - \\
    \bottomrule
    \end{tabular}
    \end{small}
\end{table}
\[\]
\textbf{Hyperparameter.} Optimization is performed using SGD:
\begin{itemize}
    \item Learning Rate: $\eta = 0.1$, Momentum: $\mu = 0.9$
    \item Weight Decay: $\lambda = 5 \times 10^{-4}$,max norm of gradient=5
    \item batch size: 32
\end{itemize}

\subsection{Shakespeare (LEAF) Experimental Setup}
For the Shakespeare dataset, we utilize a stacked Character-Level Long Short-Term Memory (CharLSTM) network for next-character prediction. The architecture maps discrete input character tokens into a dense continuous space using an embedding layer. This is followed by a 2-layer LSTM feature extractor configured with a hidden state dimension of 256 to capture long-term sequential dependencies. Finally, the unrolled sequence outputs are passed through a fully connected linear layer to predict the probability distribution over the entire vocabulary space. The recurrent hidden and cell states are zero-initialized at the start of each batch.

\begin{table}[h]
    \centering
    \caption{\textbf{CharLSTM Network Architecture.} The network processes variable-length sequences. The vocabulary size is denoted by $|V|$. The LSTM layers output the full sequence (\texttt{batch\_first=True}), which is then linearly mapped to the vocabulary size.}
    \label{tab:charlstm_shakespeare_arch}
    \vspace{2mm}
    \begin{small}
    \begin{tabular}{lccc}
    \toprule
    \textbf{Layer Type} & \textbf{In Dimension} & \textbf{Out Dimension} & \textbf{Details} \\
    \midrule
    Embedding & $|V|$ & 256 & - \\
    LSTM & 256 & 256 & 2 Layers \\
    Linear (Classifier) & 256 & $|V|$ & - \\
    \bottomrule
    \end{tabular}
    \end{small}
\end{table}
\[\]
\textbf{Hyperparameter.} Optimization is performed using SGD:
\begin{itemize}
    \item Sequence Length: $50$ batch size: 32
    \item Learning Rate: $\eta = 0.01$, Momentum: $\mu = 0.9$
    \item Weight Decay: $\lambda = 1 \times 10^{-4}$,max norm of gradient=1
    
\end{itemize}
\subsection{Compute resources} 
\label{running_needs}
Due to varying computational demands across our datasets, the experiments were conducted using two distinct hardware setups. The MNIST experiments were executed locally on a machine equipped with an Intel Core i7-1165G7 CPU (2.80GHz) and 16GB of RAM, requiring approximately 6 hours of compute time per experimental run. Conversely, due to the increased complexity of the image and text data, the CIFAR-100 and LEAF Shakespeare experiments were executed on a dedicated workstation equipped with an NVIDIA GeForce RTX 4090 GPU (24GB VRAM). Each of these larger experimental runs required approximately 14 hours of compute time to complete.

\section{Data Partitioning Algorithms}
\label{sec:appendix_partitioning}

In this section, we detail the data partitioning procedures employed in Section~\ref{sec:Experiments} to simulate realistic non-IID distributions. We utilize a noise parameter $\epsilon$ to control the degree of stochastic heterogeneity, ensuring that even clients with similar distributions possess varying data volumes and class imbalances.

\subsection{CIFAR-100 Data Partitioning (Latent Clusters)}
\label{subsec:cifar_partition}
This algorithm is used for the CIFAR-100 experiments (Section 6.3) to simulate a \textit{Label Skew} scenario with latent community structure. Clients are grouped into latent clusters, where each cluster is responsible for a distinct subset of the 100 classes.

\textbf{Mechanism:} The initialization assigns each of the $N$ clients to one of $C$ clusters, and the 100 classes are divided disjointedly among these clusters to form the "primary signal." Crucially, to simulate realistic imperfections, we add Gaussian noise $\mathcal{N}(0, \sqrt{\epsilon})$ to the probability mass of \textbf{every} class for every client, regardless of cluster membership.

This noise addition serves two purposes:
\begin{enumerate}
    \item \textbf{Intra-Cluster Heterogeneity:} It varies the distributions of the primary classes, ensuring clients in the same cluster are not identical.
    \item \textbf{Inter-Cluster Leakage:} It introduces small, non-zero probabilities for classes \textit{outside} the client's assigned cluster. This breaks the strict disjoint assumption, creating a "background noise" level that tests the algorithm's robustness to irrelevant data.
\end{enumerate}

\begin{algorithm}[H]
\caption{CIFAR-100 Partitioning (Latent Clusters)}
\label{alg:cifar_partition}
\begin{algorithmic}[1]
\REQUIRE $N$ (Clients), $C$ (Clusters), $M$ (Clients per cluster), $\epsilon$ (Heterogeneity noise)
\ENSURE Allocation matrix $\mathbf{X} \in \mathbb{R}^{N \times 100}$
\STATE \textbf{Configuration:}
\STATE \quad $P \leftarrow 100 / C$ \COMMENT{Number of unique classes per cluster}
\STATE \quad $B \leftarrow 1 / M$ \COMMENT{Base probability mass}
\STATE \textbf{Initialize:} $\mathbf{X} \leftarrow \mathbf{0}_{N \times 100}$
\STATE \textbf{Generate Noise:} $\mathbf{N} \sim \mathcal{N}(0, \sqrt{\epsilon})$

\FOR{each client $i \in \{0, \dots, N-1\}$}
    \STATE $k \leftarrow \lfloor i / M \rfloor$ \COMMENT{Determine latent cluster index for client $i$}
    \FOR{each class $j \in \{0, \dots, 99\}$}
        \STATE \textbf{Primary Signal:}
        \IF{$k \cdot P \leq j < (k+1) \cdot P$}
            \STATE $X_{i,j} \leftarrow B$ \COMMENT{Assign base mass if class is in cluster}
        \ENDIF
        \STATE \textbf{Global Noise Injection:} 
        \STATE $X_{i,j} \leftarrow X_{i,j} + N_{i,j}$ \COMMENT{Adds noise to ALL classes (in and out of cluster)}
        \STATE \textbf{Rectify:} $X_{i,j} \leftarrow \max(0, X_{i,j})$
    \ENDFOR
\ENDFOR

\STATE \textbf{Normalize:} Perform column-wise normalization on $\mathbf{X}$ to ensure valid probability distribution.
\RETURN $\mathbf{X}$
\end{algorithmic}
\end{algorithm}

\subsection{Detailed Performance Metrics and Variance Analysis}
\label{sec:var_performance}

In all experiments, we report the maximum average performance across $N$ independent runs initialized with different random seeds. To capture the variance across these runs and ensure statistical reliability, we report 1-sigma error bars representing the sample standard deviation. 

Specifically, for a given method, let $x_i$ be the performance metric of the $i$-th run at the epoch that achieved the highest mean validation performance. The mean $\mu$ and sample standard deviation $\sigma$ are calculated as:

\begin{equation*}
    \mu = \frac{1}{N} \sum_{i=1}^{N} x_i
\end{equation*}

\begin{equation*}
    \sigma = \sqrt{\frac{1}{N-1} \sum_{i=1}^{N} (x_i - \mu)^2}
\end{equation*}
\newpage
\section{Full Plots of the Empirical Study}
\label{sec_graphs}
\subsection{CIFAR-100 Experiments}

To evaluate tracking performance and convergence behavior, we conducted experiments on the CIFAR-100 image classification task under varying degrees of data heterogeneity, controlled by $\sigma$. Figures in this section track the test accuracy over 100 epochs for $\sigma=10$ and $\sigma=0.01$, respectively. Across both heterogeneity settings, SP-CACW and FedDisco consistently achieve top-tier final test accuracy ($\sim 52-54\%$). Notably, SP-CACW maintains a highly smooth and robust convergence trajectory, completely avoiding the early-epoch volatility exhibited by algorithms like ORCAEL. In contrast, methods such as Ditto and the localized ``Only My Data'' baseline plateau significantly earlier, failing to surpass $30\%$ accuracy under severe heterogeneity ($\sigma=0.01$).

\begin{figure}[h]
    \centering
    \includegraphics[width=0.9\linewidth]{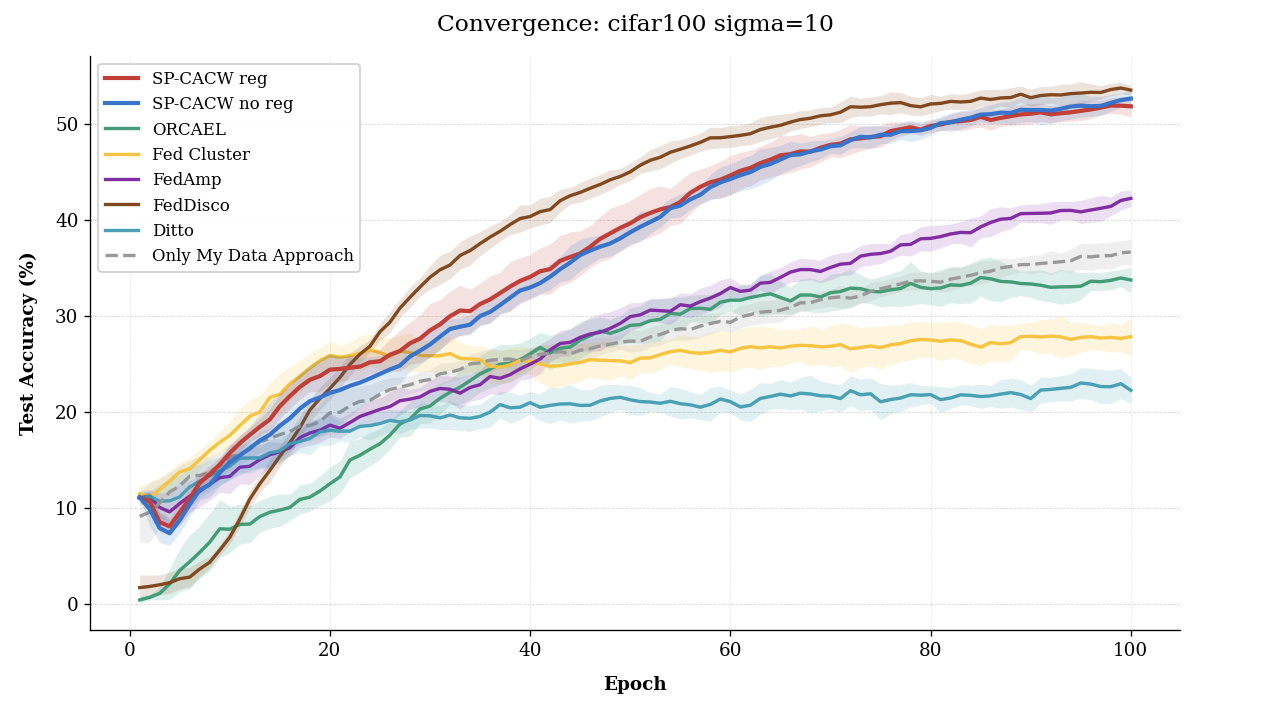}
    \label{fig:cifar_10}
\end{figure}

\begin{figure}[h]
    \centering
    \includegraphics[width=0.9\linewidth]{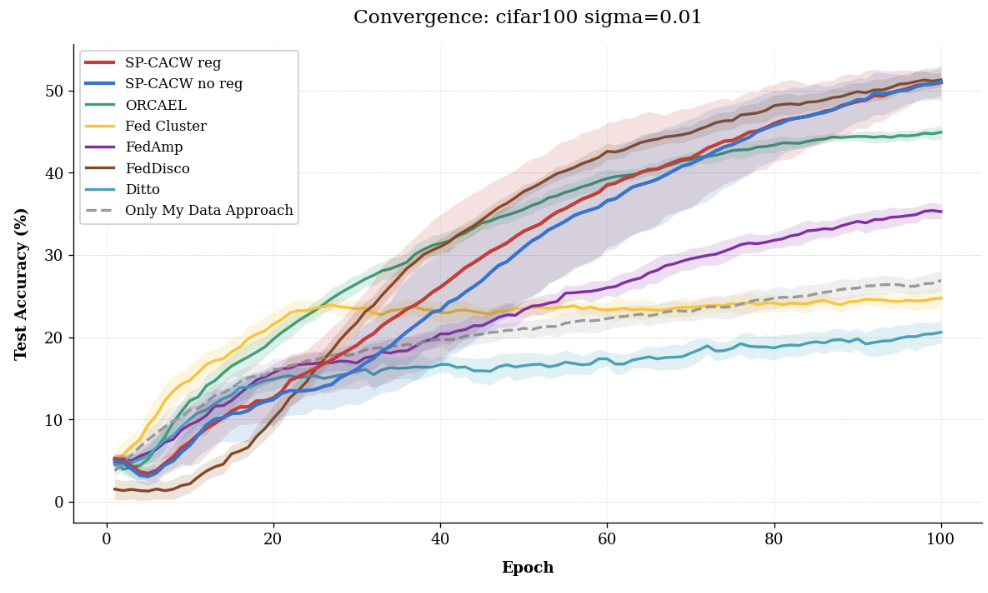}
    \label{fig:cifar_001}
\end{figure}

\FloatBarrier
\vspace{1em}

\newpage
\subsection{LEAF Shakespeare Experiments}

This subsection details the empirical results evaluated on the LEAF Shakespeare language modeling task. To accurately reflect the strict heterogeneity of realistic federated networks, the data distribution here is partitioned naturally by individual character speaking roles rather than relying on synthetic client clusters. Figures in this section track the metric value over 100 epochs across two distinct participant sampling configurations ($n=15, m=7$ and $m=30, n=3$). SP-CACW (both with and without regularization) achieves the highest final metric values ($\sim 54\%$ and $\sim 51\%$, respectively), matching FedDisco. Crucially, baseline methods such as FedAmp and Ditto suffer from significant performance gaps, while the ``Only My Data'' approach visibly degrades after epoch 40 due to severe local overfitting on the character-specific dialog.

\begin{figure}[htbp]
    \centering
    \includegraphics[width=0.9\linewidth]{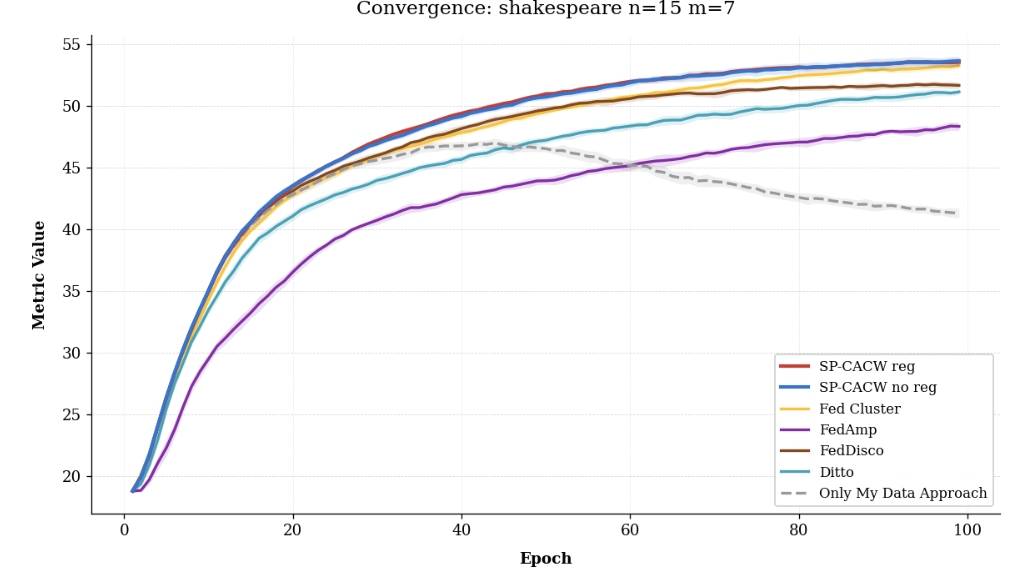}
  
    \label{fig:shak_15_7}
\end{figure}

\begin{figure}[htbp]
    \centering
    \includegraphics[width=0.9\linewidth]{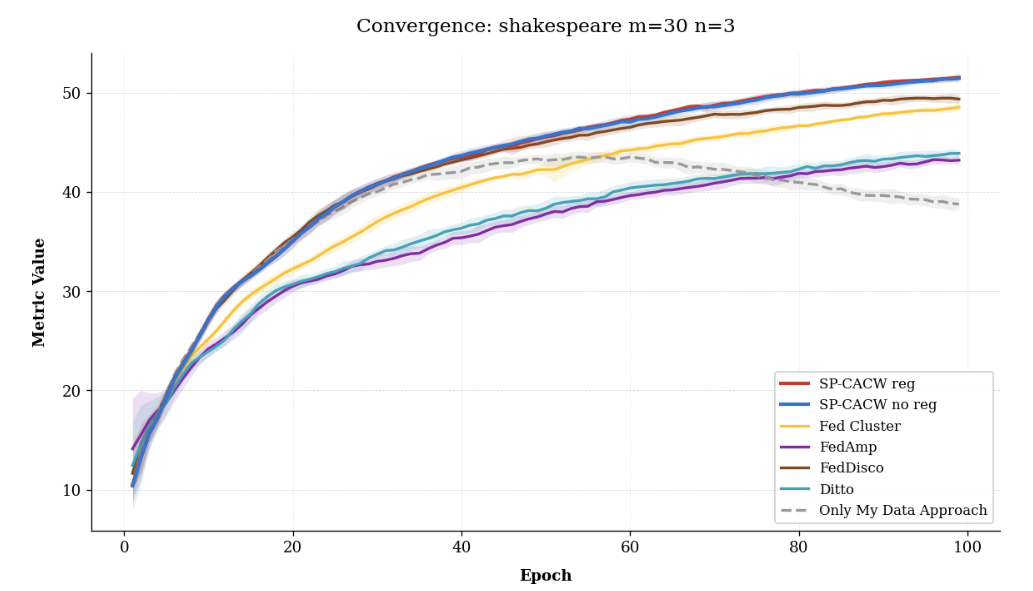}
    \label{fig:shek_30_3}
\end{figure}

\FloatBarrier
\newpage
\newpage
\subsection{MNIST Experiments}

The final set of plots showcases the algorithmic performance on the MNIST handwritten digit dataset, encompassing structural challenges such as feature rotation and label distribution skews. Figures in this sectio illustrate the test accuracy over 50 epochs for the rotation and relabeling scenarios, respectively. A critical observation in these experiments is the catastrophic failure of FedDisco; while highly competitive on standard feature skew, it completely collapses under extreme structural heterogeneity, flatlining below $55\%$ on rotation and near $30\%$ on relabeling. In stark contrast, SP-CACW demonstrates exceptional robustness, rapidly converging to the optimal $\sim 90-95\%$ accuracy band alongside ORCAEL and FedCluster. Furthermore, on the rotation task, the unregularized variant (``SP-CACW no reg'') pulls slightly ahead, suggesting that relaxing the regularizer in environments with highly orthogonal local gradients allows the active set to more aggressively discard irrelevant clients without sacrificing overall stability.

\begin{figure}[htbp]
    \centering
    \includegraphics[width=1\linewidth]{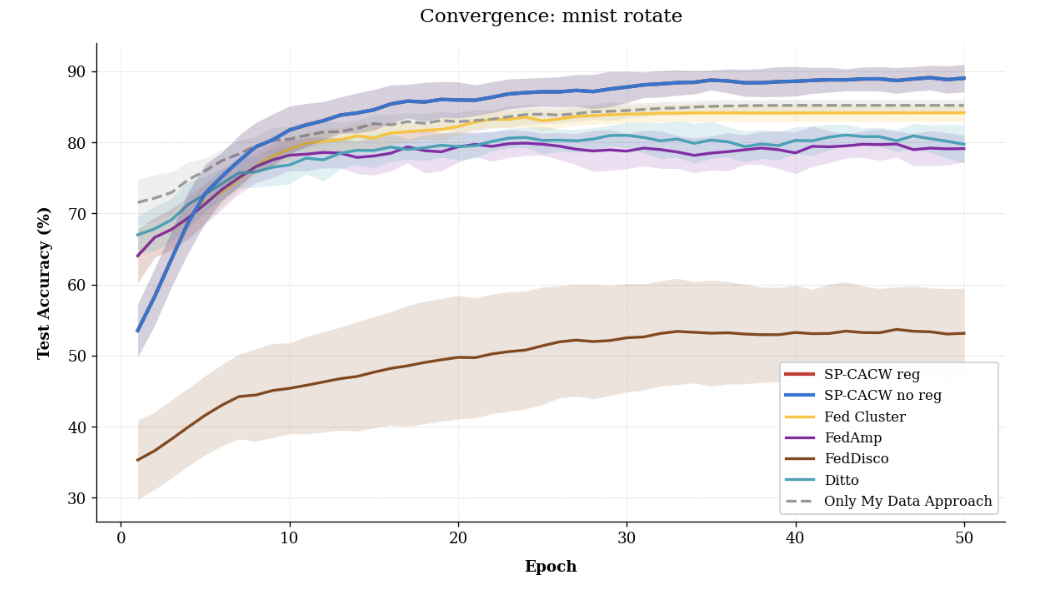}
   
    \label{fig:mnist_rot}
\end{figure}

\begin{figure}[htbp]
    \centering
    \includegraphics[width=1\linewidth]{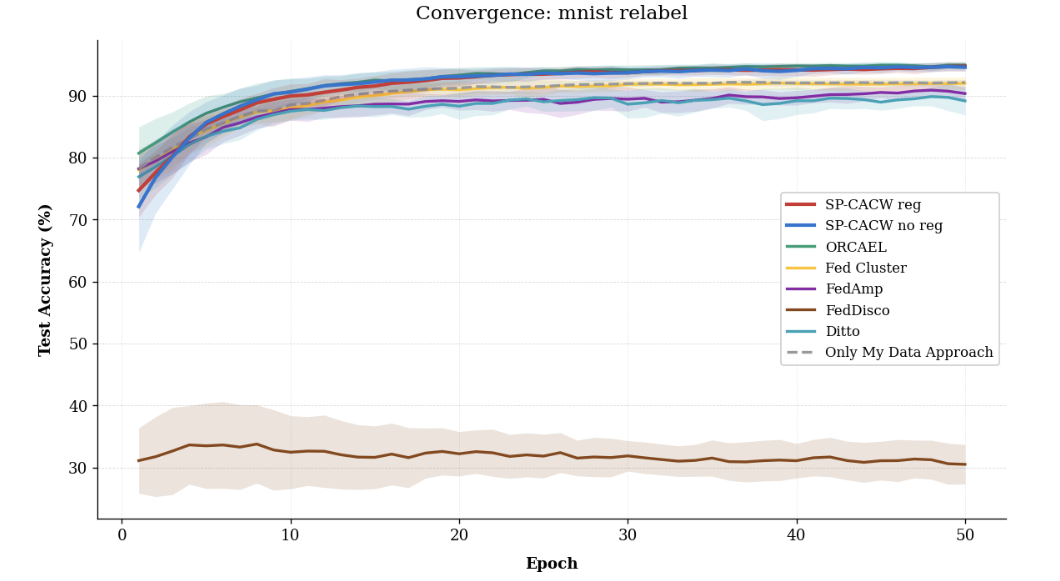}
   
    \label{fig:mnist_rel}
\end{figure}

\FloatBarrier
\newpage
\end{document}